\documentclass{article}

\usepackage{natbib}

\usepackage{bm}
\usepackage{PRIMEarxiv}
\usepackage{amsthm}
\usepackage{amssymb}
\usepackage{graphicx,amsmath,amsfonts,amssymb,bm,hyperref,url,breakurl,epsfig,epsf,color,fullpage,MnSymbol,mathbbol, fmtcount, semtrans
} 
\numberwithin{equation}{section}
\newtheorem{alg}{Algorithm}
\usepackage{titlesec}
\usepackage[]{mdframed}

\usepackage{enumitem}

\usepackage{tikz}
\usepackage{pgfplots}

\usetikzlibrary{pgfplots.groupplots}
\usepackage{amsmath,amssymb}

\titleformat{\paragraph}
{\normalfont\normalsize\bfseries}{\theparagraph}{1em}{}
\titlespacing*{\paragraph}
{0pt}{3.25ex plus 1ex minus .2ex}{1.5ex plus .2ex}

\usepackage{caption,subcaption}

\usepackage[bottom,hang,flushmargin]{footmisc} 
\usepackage[font=small]{caption}

\DeclareMathOperator*{\Ex}{\mathbb{E}}

\usepackage{hyperref}
\definecolor{darkred}{RGB}{150,0,0}
\definecolor{darkgreen}{RGB}{0,150,0}
\definecolor{darkblue}{RGB}{0,0,200}
\hypersetup{colorlinks=true, linkcolor=darkred, citecolor=darkgreen, urlcolor=black}

\newtheorem{theorem}{Theorem}
\newtheorem{lemma}{Lemma}

\newtheorem{proposition}{Proposition}
\newtheorem{definition}{Definition}
\newtheorem{assumption}{Assumption}

\newtheorem{remark}{Remark}

\newtheorem{example}{Example}

\DeclareMathOperator*{\argmin}{arg\,min}

\newcommand{\abs}[1]{\left|#1\right|}

\newcommand{\R}{\mathbb{R}}

\newcommand{\E}{\operatorname{\mathbb{E}}}

\usepackage[utf8]{inputenc} 
\usepackage[T1]{fontenc}    
\usepackage{hyperref}       
\usepackage{url}            
\usepackage{booktabs}       
\usepackage{amsfonts}       
\usepackage{nicefrac}       
\usepackage{microtype}      
\usepackage{lipsum}
\usepackage{fancyhdr}       
\usepackage{graphicx}       
\graphicspath{{media/}}     

\newcommand{\paren}[1]{\left(#1\right)}
\newcommand{\braces}[1]{\left\{#1\right\}}

\newcommand{\prf}[2]{\mathbb{P}_{#1}\left(#2\right)}

\newcommand{\indicator}[1]{\mathbb{1}\left\{ #1 \right\}}

\newcommand{\X}{\mathcal{X}}

\newcommand{\Exc}{\operatorname{\mathcal{E}_{\alpha}}}

\newcommand{\Hyp}{\mathcal{H}}

\title{Distribution-Free Rates in Neyman-Pearson Classification

}

\author{
  Mohammadreza M. Kalan \\
  Statistics, Columbia University \\
  \texttt{mm6244@columbia.edu} \\
   \And
  Samory Kpotufe \\
  Statistics, Columbia University \\
  \texttt{samory@columbia.edu} \\
}

\begin{document}
\maketitle

\begin{abstract}
We consider the problem of Neyman-Pearson classification which models unbalanced classification settings where error w.r.t. a distribution $\mu_1$ is to be minimized subject to low error w.r.t. a different distribution $\mu_0$. Given a fixed VC class $\Hyp$ of classifiers to be minimized over, we provide a full characterization of possible distribution-free rates, i.e., minimax rates over the space of \emph{all} pairs $(\mu_0, \mu_1)$. 
The rates involve a dichotomy between \emph{hard} and \emph{easy} classes $\Hyp$ as characterized by a simple geometric condition, a \emph{three-points-separation} condition, loosely related to VC dimension. 
\end{abstract}


\section{Introduction}
\emph{Neyman-Pearson classification} consists of minimizing miclassification error w.r.t. one class of data, subject to low error---i.e., below a pre-specified threshold $\alpha\in [0, 1]$---w.r.t. another class of data. The setting captures practical problems, e.g. in medical data analysis \citep{bourzac2014diagnosis,zheng2011machine}, or malware detection \citep{jose2018survey,kumar2019edima}, where a detection rule for a disease or malware is to be learned from data while maintaining low false detection rate below some threshold $\alpha$. We are interested in the statistical limits of this problem in a \emph{distribution-free} setting, as formalized below. 

Formally, consider two distributions $\mu_0, \mu_1$ on a measurable space $\paren{\X, \Sigma}$, representing classes $0$ and $1$, and a hypothesis class $\Hyp$ of decision rules $h: \X \mapsto \braces{0, 1}$, where $h(x) = 0$ or $1$ designates, respectively, whether $x$ is generated by $\mu_0$ or $\mu_1$. Neyman-Pearson classification consists of learning from data, a rule $h\in \Hyp$ that minimizes the $\mu_1$-risk $R_{\mu_1} (h) \doteq \prf{\mu_1}{h=0}$ subject to small $\mu_0$-risk $R_{\mu_0} (h) \doteq \prf{\mu_0}{h =1}$ at most some value $\alpha \in [0, 1]$. In particular, unlike in hypothesis testing, the distribution $\mu_1$ is assumed unknown, and only accessible from a finite sample; the distribution $\mu_0$ on the other hand often represents an abundant class---e.g., the population at large, lacking the disease or malware to be detected---and may be assumed known, or approximately known via a large separate sample. 
Thus, given an i.i.d. sample of size $n$ from $\mu_1$ (and perhaps a separate large sample from $\mu_0$ when it is also unknown), the learner is to return a rule $\hat h \in \Hyp$ s.t. $R_{\mu_0} (\hat h) \leq \alpha$, and whose performance is to then be assessed via its $\mu_1$-excess-risk 
$$\Exc(\hat h) \doteq R_{\mu_1} (\hat h) - \inf_{h \in \Hyp: R_{\mu_0}(h) \leq \alpha} R_{\mu_1}(h).$$ 

Our aim is to characterize all possible regimes of rates 
$\Exc(\hat h)$, as a function of $n$, in a \emph{distribution-free} setting, i.e., in a minimax sense over all pairs of distributions $\paren{\mu_0, \mu_1}$ for a fixed class $\Hyp$. This distribution-free setting has been popularized in the machine learning literature, often under the umbrella term of \emph{PAC-learning}, and similar to works on \emph{model-mispecification} in Statistics, reflects the ideal of imposing no further distributional assumption beyond the inherent desire that $\Hyp$ contains a good decision rule for the problem.


We adopt a common assumption that $\Hyp$ has finite \emph{Vapnik-Chervonenkis} (VC) dimension, which in particular allows for estimating the infimum of $R_{\mu_1}$ or $R_{\mu_0}$ from data, with no assumption on $\mu_1$ or  $\mu_0$, since it implies uniform concentration of empirical risks. Thus, our object of study is $\inf_{\hat h}\sup_{(\mu_0, \mu_1)} \E \Exc(\hat h)$ over any learner $\hat h$ mapping samples to 
a fixed VC class $\Hyp$, or a subset of $\Hyp$ depending on $\alpha$. We adopt a general setting with no further assumption on the measurable space $(\X, \Sigma)$. Our results are as follows, ignoring $\log n$ terms: 
\begin{itemize}[left=0pt] 
\item Assuming $\mu_0$ is known, we derive minimax rates for this problem for any VC class $\Hyp$. The rates turn out to depend on a dichotomy between easy and harder classes $\Hyp$, as determined by a simple geometric condition which we term a \emph{three-points-separation condition}. When $\Hyp$ satisfies this condition, minimax rates are of the familiar form $\tilde \Theta(n^{-1/2})$; to the best of our knowledge, our work provides the first lower-bound of this order for a VC class $\Hyp$, in the context of Neyman-Pearson classification. In fact, three-point-separation first arises in the construction of such an $\Omega(n^{-1/2})$ lower-bound, where it becomes apparent that any valid construction must satisfy the condition. 

When $\Hyp$ does not satisfy the condition, Neyman-Pearson classification is easy: restricting attention to typical choices $\alpha<1/2$, the subclass 
$\Hyp_\alpha(\mu_0) \doteq \braces{h \in \Hyp: R_{\mu_0}(h) \leq \alpha}$ is then highly structured irrespective of $\mu_0$, namely, it is either a singleton, or it admits a total order. As a consequence, we can show that minimax rates are strictly faster, i.e., always $ o(n^{-1/2})$. In particular, the problem is trivial if in addition, 
$\Hyp_\alpha(\mu_0)$ admits a \emph{maximal} element; whenever there is no maximal element, the minimax rate is $\tilde \Theta(n^{-1})$. 

We remark that such dichotomy in rates stands in contrast to traditional classification where, {whenever $\Hyp$ is not a singleton}, distribution free rates are never faster than 
$O(n^{-1/2})$ [see, e.g. Theorem 6.7 of \cite{shalev2014understanding}]; rather, \emph{fast} rates below $o(n^{-1/2})$ in traditional classification depend on the interaction between $\Hyp$ and the data distribution via noise conditions, rather than on the structure of $\Hyp$ \cite[see e.g.,][]{mammen1999smooth,bartlett2006empirical,koltchinskii2006local,massart2006risk}. 
\item When $\mu_0$ is unknown and only approximated from data, a similar dichotomy is present, and is still determined by three-points-separation. However delicate nuances arise due to the fact that the learner $\hat h$ may fall outside of the set $\Hyp_\alpha(\mu_0)$---since this set is now only approximately known---and therefore could have strictly lower $\mu_1$-risk than the $\mu_1$-risk minimizer over $\Hyp_\alpha(\mu_0)$ which it is evaluated against. Yet, as we show, the problem remains just as hard with similar lower-bounds, but now with some regimes of rates determined by how well certain structural aspects of subclasses of the form $\Hyp_{\alpha'}(\mu_0)$ are preserved by their empirical estimates $\Hyp_{\alpha'}(\hat \mu_0)$. These are discussed in detail in Section \ref{sec:Unknown}.
\end{itemize} 

Three-points-separation loosely relates to VC dimension, as any class of VC dimension at least 3 always satisfies the condition, while it is easily shown that classes of VC dimension 1 or 2 may or may not satisfy the condition. Some such classes not satisfying three-points-separation are induced by the classical Neyman-Pearson lemma on universally optimal decision rules for the problem (as explained in the background section below). 

Our upper and lower minimax bounds are tight up to $\log n$ terms. Our work leaves open how this might be further tightened, as it appears rather challenging. We note that even for vanilla classification with VC classes, such a question was only recently resolved after decades of work on the subject \citep{hanneke2016optimal}.



\subsection*{Further Background and Related Work}
As alluded to so far, Neyman-Pearson classification is related to hypothesis testing, which corresponds to the case where both $\mu_0$ and $\mu_1$ are known, and where $h(X)$ then denoting a test on the sample $X$, with the null-hypothesis that $X\sim \mu_0$. The optimal such test at level $\alpha$, i.e., minimizing $R_{\mu_1}(h)$ over $h$ s.t. $R_{\mu_0} (h) \leq \alpha$, is then given by the classical Neyman-Pearson lemma: under mild regularity, $\Exc(h)$, or equivalently, $R_{\mu_1}(h)$, is minimized, over all measurable $h$,  
by thresholding a density-ratio $d \mu_1/ d \mu_0$.
In other words, universally optimal decision rules for a given pair $(\mu_0, \mu_1)$ is given by the class of level sets $\Hyp^* \doteq \braces{\indicator{d \mu_1/ d \mu_0 \geq \lambda}: \lambda >0}$, under mild regularity conditions \cite[see][Theorem 3.2.1]{lehmann1986testing}. For ease of discussion, we will refer to $\Hyp^*$ as a \emph{Neyman-Pearson} class (for $(\mu_0, \mu_1)$).

Such relation to hypothesis testing gives rise naturally to nonparametric solutions to the problem. Namely, given samples from both $\mu_1$ and $\mu_0$ (unknown), various approaches have been proposed to estimate the density-ratio $d \mu_1/ d \mu_0$ to be used as a plug-in estimate of the universally optimal rule at each level $\alpha \in [0, 1]$ (see e.g., \citep{tong2013plug, zhao2016neyman, tong2018neyman,tian2021neyman}). Furthermore, when $\mu_0$ is unknown, one would also estimate an adequate level-set of the density-ratio, commensurate with the desired level $\alpha$. Such nonparameteric density-ratio and level-set estimation are difficult problems and can in fact be infeasible in practice, especially in high-dimensional settings with limited data. For instance, under smoothness conditions on the density ratio, rates of convergence, in the worst-case, are of the form $n^{-1/O(d)}$, i.e., are exponentially slow in dimension $d$. However, as shown in \citep{tong2013plug,zhao2016neyman} the problem can benefit from \emph{margin conditions}: these are conditions on $(\mu_0, \mu_1)$ that characterize easier problems where the density ratio $d \mu_1/ d \mu_0 $ concentrates away from the optimal threshold $\lambda = \lambda(\alpha)$. 
Under such conditions rates faster than $o(n^{-1/2})$ can be obtained even in the nonparametric setting. This attests to the fact that Neyman-Pearson classification is easy for some pairs of distributions $(\mu_0, \mu_1)$, which in hindsight is evident when one considers, e.g., the extreme case where $\mu_0$ and $\mu_1$ have non-overlapping support---notwithstanding the fact that the density ratio is ill-defined in this extreme case. In the present work however, rather than conditions on $(\mu_0, \mu_1)$ that influence regimes of rates, we elucidate new conditions pertinent to the hypothesis class $\Hyp$ being optimized over, that separates regimes of rates, irrespective of given distributions. 

The idea of fixing a hypothesis class to optimize over, in contrast to the nonparametric approaches discussed above, has it roots in many early research work that aimed at structural assumptions that could lead to more practical procedures \cite[see][]{casasent2003radial, cannon2002learning, scott2005neyman, scott2007performance, han2008analysis, rigollet2011neyman, tong2020neyman, ma2020quadratically}. Typically, such structural assumptions consist of either parametric models on $(\mu_0, \mu_1)$---thus restricting the class $\Hyp$ of rules to optimize over---or alternatively $\Hyp$ itself may be directly modeled, however, with somewhat conflicting desiderata: on one side $\Hyp$ ought to be rich enough to yield good rules for a given $(\mu_0, \mu_1)$, but also has low enough complexity (usually bounded VC dimension) to allow for empirical approximation of relevant errors. Such delicate tradeoff on the choice of $\Hyp$ is not the subject of this work, but certainly comes to mind as we characterize a new dichotomy in achievable rates.  


Existing upper-bounds on $\Exc$, starting with the seminal works of \citep{cannon2002learning, scott2005neyman, scott2007performance, blanchard2010semi}, are of the form $\tilde O(n^{-1/2})$. However, we know of no corresponding lower-bound in the literature, and show here that in fact such a rate corresponds to classes $\Hyp$ that satisfy three-points-separation. 

The bulk of theoretical works on Neyman-Pearson classification has focused on more practical aspects of the problem. For instance,  
\cite{han2008analysis, rigollet2011neyman, ma2020quadratically, lin2023gbm} consider minimizing convex surrogates of $R_{\mu_1}$ over low complexity classes $\Hyp$ (e.g., linear models, convex hull of a finite class, general VC classes, etc). \cite{tian2021neyman} analyze general optimization frameworks that exploit dual objectives to the Neyman-Pearson problem and establish consistency under both nonparametric and parametric models. \cite{scott2019generalized} establishes consistency in settings with imperfect but related data. More recently, \cite{fan2023neyman, zeng2022bayes} draw links between Neyman-Pearson classification and emerging \emph{fairness} constraints between populations in machine learning applications. As such, the basic question addressed in the present work, namely that of characterizing regimes of rates for the problem, has so far remained open.

\section{Setup and Definitions}

We consider a measurable space $(\X, \Sigma)$ and a hypothesis class $\Hyp$ of measurable functions 
$h: \X \mapsto \braces{0, 1}$. 

\begin{assumption} Without loss of generality, we also assume that for every $x\in \X$, the set $\braces{x} \in \Sigma$: equivalently, since every $h\in \Hyp$ is measurable and therefore constant on atoms of the $\sigma$-algebra $\Sigma$, we can simply take $\X$ as the set of atoms of $\Sigma$ in all discussions. 
\end{assumption}

{\bf Neyman-Pearson Classification.} As introduced earlier, we consider any two probability distributions $\mu_0, \mu_1$ on $(\X, \Sigma)$ and recall the $\mu_0$ and $\mu_1$ risks $R_{\mu_0}(h)\doteq\mathbb{P}_{\mu_0}(h=1)$ and $R_{\mu_1}(h)\doteq\mathbb{P}_{\mu_1}(h=0)$. Also, for $0< \alpha \leq 1$, we retain the notation $\Hyp_{\alpha}(\mu_0)\doteq \braces{h\in \Hyp: R_{\mu_0}(h) \leq \alpha}$. The Neyman-Pearson classification problem at level $0< \alpha \leq 1$ consists of minimizing $R_{\mu_1}(h)$ over $h \in \Hyp_{\alpha}(\mu_0)$, given a sample from $\mu_1$ and $\mu_0$ when $\mu_0$ is also unknown. 

When $\mu_0$ is unknown, it is typical to relax the constraint $\Hyp_\alpha(\mu_0)$ to $\Hyp_{\alpha + \epsilon_0}(\mu_0)$, for some slack $\epsilon_0$ depending on $\mu_0$ sample size. In such a case, note that it is possible that the infimum of $\mu_1$-risk over $\Hyp_{\alpha + \epsilon_0}(\mu_0)$ is strictly smaller than the infimum of the $\mu_1$-risk over its subset $\Hyp_{\alpha}(\mu_0)$; we therefore define the excess risk as follows to cover both the case where the learner $\hat h$ maps to 
$\Hyp_\alpha(\mu_0)$ where $\mu_0$ is known, or to $\Hyp_{\alpha + \epsilon_0}(\mu_0)$. In all that follows, we abuse notation and let $\hat h$ denote both the learner, and the function it maps to in $\Hyp$.

\begin{definition}\label{excess}
    Let $ 0< \alpha \leq 1$. We define the {\bf $\mu_1$-excess-risk}, at level $\alpha$,  of an $\hat{h}\in\mathcal{H}$ as follows:
    \begin{align*}
    \mathcal{E}_{\alpha}(\hat{h})=\max\bigg\{0,R_{\mu_1}(\hat{h})-\inf_{h \in \Hyp_{\alpha}(\mu_0)} R_{\mu_1}(h)\bigg\}.
    \end{align*}
\end{definition}

{\bf Empirical Risks.} We will refer to the empirical version of the  $\mu_1$-risk, computed on an i.i.d. sample $S_{\mu_1}\sim \mu_1^n$ as $\hat R_{\mu_1}(h) \doteq \frac{1}{n}\sum_{X\in S_{\mu_1}} \indicator{h(X) \neq 1}$. Similarly, define $\hat R_{\mu_0}\doteq \frac{1}{n_0}\sum_{X\in S_{\mu_0}} \indicator{h(X) \neq 0}$ over $S_{\mu_0}\sim \mu_0^{n_0}$. 

As stated earlier, we adopt the assumption that the hypothesis class $\Hyp$ has finite VC dimension as defined below.

\begin{definition}[\citep{vapnik2015uniform}]
A hypothesis class $\mathcal{H}$ {\bf shatters} a finite set of points $x_1,...,x_m\in \mathcal{X}$ if any possible $0$-$1$ labeling of these points is realized by some $h \in \mathcal{H}$. The {\bf VC dimension} of $\mathcal{H}$, denoted $d_{\mathcal{H}}$, is then defined as the largest number of points that can be shattered by $\mathcal{H}$. 
\end{definition}

As alluded to so far, $\Hyp$ is referred to as a VC class whenever $d_\Hyp < \infty$. 
As is well known, VC classes $\Hyp$ satisfies $\sup_{\mu_1} \sup_{h \in \Hyp} \abs{\hat{R}_{\mu_1}(h) - R_{\mu_1}(h)} = o_P (1)$, i.e., they admit \emph{distribution-free} uniform convergence of empirical risks, hence their obvious relevance to the setting.

{\bf Asymptotic notation.} We will use the shorthand notations $\asymp$ and $\lesssim$ to indicate respectively, the order of a minimax rate, and inequality, up to  constants and $\log{n}$ terms.



\section{Main Results}
We present our main results in this section, for both the settings of known and unknown $\mu_0$ . 
We start by defining a few concepts which turn out critical in characterizing the minimax rates for the problem. 

\begin{definition}
We say that a subset $\tilde{\mathcal{H}}$ of $\cal H$ \emph{\bf admits a maximal element}, if $\exists h \in \tilde{\mathcal{H}}$ such that $\forall h'\in \tilde{\mathcal{H}}$ we have $\braces{h = 1} \supset \braces{h' = 1}$. Such an $h$ will be referred to as \emph{the maximal element} of $\tilde \Hyp$. 
\end{definition}

Next, we turn to the most important definition of this section. 
\begin{definition}\label{def:three_points}
We say that $\cal H$ {\bf separates three points}, alternatively, satisfies {\bf three-points-separation}, if there exist two hypotheses $h_1, h_2 \in \cal H$, and three points $x_0, x_1, x_2$ in $\X$ such that the following two conditions holds: 
\begin{itemize} 
\item[\rm (a)] $h_1 (x_0) = h_2(x_0) = 0$,
\item[\rm (b)] $h_1(x_1) = h_2(x_2) = 1$ and $h_1(x_2) = h_2(x_1) = 0$. 
\end{itemize}

Alternatively, we say that two sets $\X_1, \X_2 \subset \X$ separate three points, if we have that (a) $\X_1 \cup \X_2 \neq \X$ and (b) $\X_1 \not\subset \X_2$ and $\X_2 \not\subset \X_1$; thus, $\Hyp$ separates three points if $\exists h_1, h_2\in \Hyp$ such that $\braces{h_1 =1}$ and $\braces{h_2 =1}$ separate three points. 
\end{definition}

The condition is illustrated in Figure \ref{fig:sub1_sep_def}. It should be immediately clear that any $\Hyp$ of VC dimension $d_\Hyp \geq 3$ satisfies three-points-separation since by definition it \emph{shatters} at least three points. However, VC dimension is only loosely related as $d_\Hyp < 3$ alone does not negate the condition; for emphasis, we have the following simple proposition, substantiated by the remark thereafter. 

\begin{proposition}\label{exis}
For any value of VC dimension $d_\Hyp$ in $\braces{1, 2}$, there exist hypothesis classes $\Hyp$ that satisfy three-points-separation and some that do not. 
\end{proposition}
\begin{remark}[Examples for $d_\Hyp < 3$.]\label{remark2}
As a basic example, any \emph{nested} class $\Hyp$, e.g., one-sided thresholds $\braces{\indicator{x \geq \lambda}: \lambda \in \R}$ as in Figure \ref{fig:sub2_gaussian}, cannot satisfy three-points-separation since (b) in Definition \ref{def:three_points} cannot hold. As a related example of a nested class, arguably impractical, the Neyman-Pearson class $\braces{\indicator{d\mu_1/d\mu_0 \geq \lambda}: \lambda\in \R}$ on a general space $\X$ cannot separate three points, even if it may induce rich subsets $\braces{h =1}$ of $\X$. Furthermore, it is well known that any such class $\Hyp$ of hypotheses ordered by inclusion has VC dimension $d_\Hyp =1$. On the other hand, a simple class of VC dimension $1$ that separates three points is given by $\Hyp = \braces{h = \indicator{x}: x \in \X}$, i.e., where every $h$ is $1$ exactly on one point. 

 The class of two-sided thresholds $\braces{\indicator{x \geq \lambda}, \indicator{x \leq \lambda}: \lambda \in \R}$ is a simple of $\Hyp$ with $d_{\Hyp} =2$ which separates three points. The case of a class with $d_\Hyp =2$ which does not separate three points is a bit less trivial. A simple example can be constructed as follows. Let $\Hyp_0 = \braces{h_\lambda= \indicator{x \geq \lambda}: \lambda \in [x_0, \infty)]}$, i.e., one-sided thresholds on the half line $\X \doteq [x_0, \infty)$; then let $\Hyp = \Hyp_0 \cup \braces{h}$ where 
 $h \doteq \indicator{x \neq x_1}$ for some $x_1 > x_0$, i.e., $h$ is $0$ only at $x_1$. Then, as previously discussed, no two hypotheses in $\Hyp_0$ separates three points, by nestedness; now $h$ and $h_\lambda \in \Hyp_0$ cannot separate three-points either: if $\lambda > x_1$, then $\{h_\lambda = 1\} \subset \{h =1\}$, while $\lambda \leq x_1$ implies $\{h_\lambda = 1\} \cup \{h =1\} = \X$. This also implies that $d_\Hyp < 3$, while $x_0, x_1$ are clearly shattered.

\end{remark}


\begin{figure}
\centering
\begin{subfigure}{.46\textwidth}
  \centering
  \includegraphics[width=.6\textwidth]{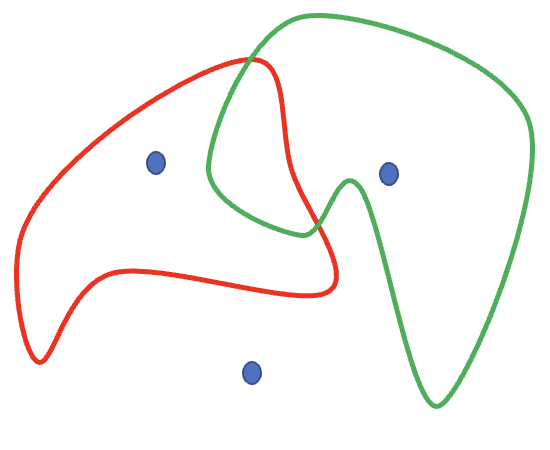}
  \caption{Two sets separating three points.}
  \label{fig:sub1_sep_def}
\end{subfigure}
\hfill
\hspace{-60mm}
\begin{subfigure}{.46\textwidth}
  \centering
\includegraphics[width=1\textwidth]{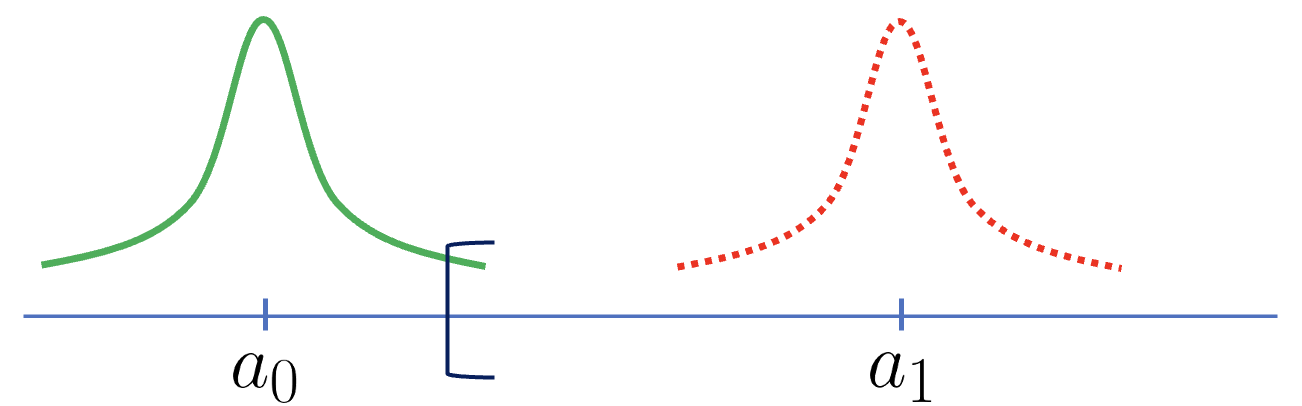}
  \caption{The class $\mathcal{H}=\{\mathbb{1}\{x\geq \lambda\}: \lambda\in \mathbb{R}\}$ of one-sided thresholds does not separate three points. 
  Such $\Hyp$ is induced, e.g., by Gaussian modeling assumptions of the form $\{(\mu_0,\mu_1)=(\mathcal{N}(a_0,\sigma^2),\mathcal{N}(a_1,\sigma^2)): a_1>a_0\}$.} 
  \label{fig:sub2_gaussian}
\end{subfigure}
\caption{}
\label{fig:test}
\end{figure}


\subsection{Known $\mu_0$, Exact Level $\alpha$}
We first consider a setting where, while $\mu_1$ is unknown, the learner has knowledge of the distribution $\mu_0$. This models, in the extreme, settings where much knowledge of $\mu_0$ is available; even more importantly, the setting is of intellectual interest as it isolates the hardness of the problem as due solely to the lack of knowledge of $\mu_1$. Also, importantly, the insights of this section serve as a building block to the more difficult case of unknown $\mu_0$ considered in Section \ref{sec:Unknown}. 

\begin{definition}\label{Distribution-free-learner} 
We call $\hat h$ an {\bf exact $\alpha$-learner} if it maps any i.i.d. sample $S_{\mu_1}\sim \mu_1^n$, $\mu_1$ unknown, to a function in $\mathcal{H}_{\alpha}(\mu_0).$
\end{definition}







We note that \emph{knowledge of $\mu_0$} is formalized in the theorem below by considering sub-families $\mathcal{U}(\mu_0)$ with fixed $\mu_0$, but varying $\mu_1$.  

\begin{theorem}\label{thm1}
Let $\cal H$ be a hypothesis class with finite VC dimension $d_{\cal H}$. Fix any $0< \alpha < \frac{1}{2}$. Let $\cal U$ denote all pairs of probability distributions $(\mu_0, \mu_1)$ over 
$(\mathcal X, \Sigma)$ such that the set $\mathcal{H}_{\alpha}(\mu_0)$ is non-empty. In all that follows, we let $\hat h$ denote any exact $\alpha$-learner as in Definition \ref{Distribution-free-learner}, operating on $S_{\mu_1}\sim \mu_1^n$. Assume $n$ is large enough so that $\frac{d_{\mathcal{H}}}{n}\leq \frac{1}{2}$. Furthermore, let ${\cal U}(\mu_0)$ denote the projection of $\cal U$ along a fixed $\mu_0$.

\begin{enumerate}
\item[\rm (i)] Suppose $\cal H$ separates three points. Then there exists $\mu_0$ such that
\begin{align*}
   \inf_{\hat{h}} \sup_{\mathcal{U}(\mu_0)} \Ex_{S_{\mu_1}}  \mathcal{E}_{\alpha}(\hat{h}) \asymp \sqrt{\frac{d_{\mathcal{H}}}{n}}.
\end{align*}
\item [\rm (ii)] Suppose $\cal H$ does not separate three points. The minimax rate is then of strictly lower order either $d_{\mathcal{H}}/n$ or $0$ (where a rate of $0$ indicates that the problem is trivial). 
 More specifically, we have the following.  

\begin{enumerate}

\item Consider any $\mu_0$ such that $\mathcal{H}_\alpha(\mu_0)$ does not admit a maximal element. We then have 

\begin{align*}
    \inf_{\hat h} \sup_{{\cal{U}}(\mu_0)} \Ex_{S_{\mu_1}} \ \mathcal{E}_{\alpha}(\hat h) \asymp {\frac{d_{\cal H}}{n}}.
\end{align*}
\item Consider any $\mu_0$ such that $\mathcal{H}_\alpha(\mu_0)$ admits a maximal element. We then have 
\begin{align*}
    \inf_{\hat h} \sup_{{\cal U}(\mu_0)} \Ex_{S_{\mu_1}} \ \mathcal{E}_{\alpha}(\hat h) = 0.
\end{align*}
\end{enumerate}

\end{enumerate}
\end{theorem}

For intuition on the dichotomy, three-points separation allows for the existence of sets of the form $\braces{h =1}, \braces{h'=1}$, $h, h' \in \Hyp_\alpha(\mu_0)$ whose mass under $\mu_1$ have to be compared to minimize $\mu_1$ error; the discrepancy between empirical $\hat \mu_1$ and population $\mu_1$ mass of sets is of order $n^{-1/2}$. On the other hand, when three-points separation fails, the problem roughly boils down to determining whether two sets $\braces{h =1}\supset \braces{h'=1}$, one containing the other, have the same mass under $\mu_1$; this only requires one data point in the set difference $\braces{h =1}\setminus \braces{h'=1}$, hence the $n^{-1}$ rate. More exact intuition is given in Section \ref{sec: Minimax Lower-Bounds}. 

Cases (ii.a) and (ii.b) are illustrated in the following examples. Generally, for typical problem spaces (a space being defined by $(\X, \Sigma, \Hyp)$) both (ii.a) and (ii.b) may hold (Example \ref{ex:maximal_or_no_maximal_element}). Interestingly,  there are also problems where only (ii.b) can hold, i.e., irrespective of the choice of $\mu_0$ (Example \ref{ex:always_no_maximal_element}); in other words, for such problem spaces, the rate is trivial whenever $\Hyp$ does not satisfy three-points separation and $\mu_0$ is known. 

\begin{example}\label{ex:maximal_or_no_maximal_element}
Let $(\X, \Sigma) = (\R, {\cal B}(\R))$, and $\Hyp = \braces{h_\lambda = \indicator{x\geq \lambda}: \lambda \in \R}$. Then as previously discussed, $\Hyp$ does not separate three points. Now let 
$\mu_0$ for instance denote ${\cal N}(0, 1)$, then clearly for any $0< \alpha < 1/2$, $\Hyp_{\alpha}(\mu_0)$ has a maximal element, namely at $h_{z_\alpha}$ for $z_\alpha \doteq \Phi^{-1} (1-\alpha)$. However, if we instead consider a distribution 
$\mu_0 = \frac{1}{2}{\cal N}(0, 1) + \frac{1}{2}\bm{\delta} (z_\alpha)$ (where $\bm{\delta}$ denotes Dirac delta, i.e., a point mass at $z_{\alpha}$), then 
$\Hyp_{\alpha/2}(\mu_0) = \braces{h_\lambda: \lambda > z_\alpha}$ admits no maximal element.   
\end{example}

\begin{example}\label{ex:always_no_maximal_element} Now consider $\mathcal{X} = \mathbb{N}$, and again a class of one-sided thresholds $\mathcal{H}=\{h_i(x)=\mathbb{1}\{x\geq i\}: i\in \mathbb{N}\}$. Then it should be clear that for every $0< \alpha < 1/2$, for any $\mu_0$, $\Hyp_\alpha(\mu_0)$ has a maximal element. In fact such an example can be extended to any totally ordered discrete set $\X$ and extending $\Hyp$ according to the ordering. 
\end{example}


\subsection{Unknown $\mu_0$, Approximate Level $\alpha$}\label{sec:Unknown}
We now consider the case where both $\mu_0$ and $\mu_1$ are unknown, and only accessible via samples $S_{\mu_0}\sim \mu_0^{n_0}$ and $S_{\mu_1}\sim \mu_1^n$. We will focus on the most common setting where the learner is allowed some \emph{slack}, i.e., may return a hypothesis from $\Hyp_{\alpha + \epsilon_0}$ for some $\epsilon_0$ depending on $n_0$. 

We note that the minimax lower-bounds of Theorem \ref{thm1} hold immediately for the less-studied case where the learner is still to return a hypothesis in $\Hyp_{\alpha}(\mu_0)$ using $S_{\mu_0}$: this is evident from the fact that an exact $\alpha$-learner from that theorem may always sample from $\mu_0$ since it has full knowledge. 


We now formalize the types of learners considered in this section. 

\begin{definition}\label{Distribution-free-learner-delta} Let $0 < \epsilon_0$, $\delta_0 < 1$.  
We call $\hat h$ an {\bf $(\epsilon_0,\delta_0)$-approximate $\alpha$-learner} if it maps any two independent i.i.d. samples $S_{\mu_0}\sim \mu_0^{n_0}$ and $S_{\mu_1}\sim \mu_1^{n}$, $\mu_0, \mu_1$ unknown, to a function in 
$\mathcal{H}_{\alpha+\epsilon_0}(\mu_0)$
with probability $1-\delta_0$ w.r.t. $ S_{\mu_0}, S_{\mu_1}$.
\end{definition}

In particular, we will consider such learners that are given a minimal amount of samples to achieve the $\epsilon_0$ slack. To this end, we need the following definition. 

\begin{definition}[Sampling requirements]\label{sample}
Fix a hypothesis class $\mathcal{H}$ with finite VC dimension, and let $0< \epsilon_0,\delta_0<1$. Then, $n_0(\epsilon_0,\delta_0)$, depending on $\Hyp$, will denote 
the minimum sample size $n_0$ such that we have
$$\sup_{\mu_0}\underset{S_{\mu_0}}{\mathbb{P}}\paren{\sup_{h\in \mathcal{H}}|\hat{R}_{\mu_0}(h)-R_{\mu_0}(h)|>\epsilon_0}\leq \delta_0.$$
    
\end{definition}
Lemma \ref{relative_vc_lemma} in Appendix gives an upper bound of $\tilde{\mathcal{O}}((d_{\mathcal{H}}+\log{1/\delta_0})/{\epsilon_0^2})$ for $n_0(\epsilon_0,\delta_0)$. Note that although $(\epsilon_0, \delta_0)$-approximate $\alpha$-learner is allowed to return a hypothesis from $\mathcal{H}_{\alpha+\epsilon_0}(\mu_0)$, it is still evaluated against the best hypothesis in $\mathcal{H}_{\alpha}(\mu_0)$. This learner, although weaker in the sense of not knowing $\mu_0$, can also be more powerful than exact $\alpha$-learner in the sense that the returned hypothesis from $\mathcal{H}_{\alpha+\epsilon_0}(\mu_0)$ might have strictly smaller error than the $\mu_1$-risk minimizer over $\mathcal{H}_{\alpha}(\mu_0)$. Nevertheless, similar lower-bounds hold since for some $\mu_0$'s, it holds that $\mathcal{H}_{\alpha+\epsilon_0}(\mu_0) = \mathcal{H}_{\alpha}(\mu_0)$.

However, the conditions distinguishing $n^{-1}$ rates from trivial $0$ rates are now different. This is because a main difficulty in establishing matching upper and lower bounds have to do with whether properties of $\Hyp_{\alpha + \epsilon_0}(\mu_0)$, e.g., existence of a maximal element, extend to the empirical counter-parts $\Hyp_{\alpha + \epsilon_0}(\mu_0)$. For this reason, it is easier to consider the subclass of finitely-supported distributions $\mu_0$, which include empirical distributions, and the structure they induce on subclasses of $\Hyp$.

\begin{theorem}\label{thm2}
 Let $\cal H$ be a hypothesis class with finite VC dimension $d_{\cal H}$. Fix any $0< \alpha < \frac{1}{2}$, $0<\delta_0<\frac{1}{2}$, and $0<\epsilon_0\leq\frac{\alpha}{d_{\mathcal{H}}}$ satisfying $\alpha+2\epsilon_0<\frac{1}{2}$. Let $\cal U$ denote all pairs of distributions $(\mu_0, \mu_1)$ over  $(\mathcal X, \Sigma)$ such that the set $\mathcal{H}_{\alpha}(\mu_0)$ is non-empty. In what follows, we let $\hat h$ denote any $(\epsilon_0,\delta_0)$-approximate $\alpha$-learner, according to Definition \ref{Distribution-free-learner-delta}, operating on $S_{\mu_0}\sim \mu_0^{n_0}$, and $S_{\mu_1}\sim \mu_0^{n}$. We let $n_0 = n_0(\frac{\epsilon_0}{2},\delta_0)$, and assume $n$ is large enough so that $\frac{d_{\mathcal{H}}}{n}\leq \frac{1}{2}.$  
 
\item[\rm (i)] Suppose $\cal H$ separates three points. We then have for some universal constants $c, C$ that  
\begin{align*}
    c\cdot \sqrt{\frac{d_{\cal H}}{n}} \  \leq  \ \inf_{\hat h} \sup_{\cal U} \Ex_{S_{\mu_0},S_{\mu_1}} \ \bigg[ \mathcal{E}_{\alpha}(\hat h)\cdot\mathbb{1}\big\{\hat{h}\in \mathcal{H}_{\alpha+\epsilon_0}(\mu_0)\big\}\bigg] \ \lesssim \ C \cdot\sqrt{\frac{d_{\cal H}}{n}} +  \delta_0.
\end{align*}
\item [\rm (ii)] Suppose $\cal H$ does not separate three points. The minimax rate is then of strictly lower order either $d_{\mathcal{H}}/n$ or $0$. More specifically, the following holds. 

 \begin{enumerate}
     \item[(a)] Suppose there exists a distribution $\mu_0$ over $(\mathcal{X},\Sigma)$ with finite support, such that $\mathcal{H}_{\alpha+\epsilon_0}(\mu_0)$ does not admit a maximal element. We then have for some universal constants $c,C$ that
     \begin{align*}
    c\cdot\frac{d_{\mathcal{H}}}{n}\ \leq\ \inf_{\hat h} \sup_{{\cal{U}}} \Ex_{S_{\mu_0},S_{\mu_1}} \ \bigg[ \mathcal{E}_{\alpha}(\hat h)\cdot\mathbb{1}\big\{\hat{h}\in \mathcal{H}_{\alpha+\epsilon_0}(\mu_0)\big\}\bigg] \ \lesssim \ C\cdot{\frac{d_{\cal H}}{n}}+\delta_0.
\end{align*}
     \item[(b)] If for all distributions $\mu_0$ over $(\mathcal{X},\Sigma)$ with finite support, $\mathcal{H}_{\alpha+\epsilon_0}(\mu_0)$ admits a maximal element, then for all $(\mu_0,\mu_1)\in \mathcal{U}$ we have 
     \begin{align*}
        \inf_{\hat h}\ \mathbb{P}\bigg(\mathcal{E}_{\alpha}(\hat h)\cdot\mathbb{1}\big\{\hat{h}\in \mathcal{H}_{\alpha+\epsilon_0}(\mu_0)\big\}=0 \bigg)\ \geq\ 1-\delta_0.
     \end{align*}

 \end{enumerate}
\end{theorem}
In particular, letting $\delta_0=\mathcal{O}(1/n)$ ensures that upper and lower bounds match in (i) and (ii.a). Before we discuss the main technicality in establishing the above result, we first verify that the situations described in (ii.a) and (ii.b) are non-empty. First, notice that (ii.b) is already covered by Example \ref{ex:always_no_maximal_element}. For (ii.a) we have the following example: 

\begin{example}
Let $(\X, \Sigma) = (\R, {\cal B}(\R))$, and $\Hyp = \braces{h_\lambda = \indicator{x\geq \lambda}: \lambda \in \R}$. Then, fix any positive $0< \alpha, \epsilon_0 < 1/2$, and consider a measure $\mu_0$ on points $x_1< x_2< x_3$, assigning mass, respectively, $1-\alpha-\epsilon_0/2$, 
 $\epsilon_0/2$, and $\alpha$. Then $\mathcal{H}_{\alpha+\epsilon_0}(\mu_0)=\{h_{\lambda}=\indicator{x\geq \lambda}:\lambda>x_1\}$ does not admit a maximal element, while $\mathcal{H}_{\alpha}(\mu_0)=\{h_{\lambda}=\indicator{x\geq \lambda}: \lambda\geq x_3\}$ admits a maximal element.
\end{example}


Part (i) of the above Theorem \ref{thm2} shares the same construction as part(i) of Theorem \ref{thm1}: we make use of a $\mu_0$ such that $\mathcal{H}_{\alpha+\epsilon_0}(\mu_0)=\mathcal{H}_{\alpha}(\mu_0)$, and notice that lower-bounds for approximate learners imply the same for exact learners. See Lemma \ref{4.3} of Section \ref{sec: Minimax Lower-Bounds}.

Part (ii.a) of Theorem \ref{thm2} is established by reducing to part (ii.a) of Theorem \ref{thm1}. Namely, we show that the conditions imply the existence of a different measure $\mu_0'$ such that $\mathcal{H}_{\alpha}(\mu'_0)=\mathcal{H}_{\alpha+\epsilon_0}(\mu'_0)$ which also does not admit a maximal element. See Lemmas \ref{lower} and \ref{lemma4.8} of Section \ref{sec: Minimax Lower-Bounds}.

Part (ii b) of Theorem \ref{thm2} is established by showing the existence of a learner achieving $0$ error with high probability. The learner in question relies on the empirical set $\mathcal{H}_{\alpha+\epsilon_0/2}(\hat{\mu}_0)$, using the fact that, with high probability, 
$\Hyp_{\alpha}(\mu_0) \subset \mathcal{H}_{\alpha+\epsilon_0/2}(\hat{\mu}_0) \subset \Hyp_{\alpha + \epsilon_0}(\mu_0)$. Thus, it suffices to show that, under the condition (ii.b), for \emph{any} $\mu_0$, the set $\mathcal{H}_{\alpha+\epsilon_0/2}(\hat{\mu}_0)$ also admits a maximal element which the learner returns. See Lemma \ref{upp-not_sep_1} of Section \ref{sec: Generic Upper-Bounds}.
\section{Overview of Analysis}\label{sec:4}
Sections \ref{sec: Generic Upper-Bounds}, \ref{sec: Minimax Lower-Bounds} provide supporting results for Theorems \ref{thm1} and \ref{thm2}, proved in Sections \ref{sec: proof:theorem 1} and \ref{sec: Proof: theorem 2}. 

\subsection{Supporting Upper-Bounds}\label{sec: Generic Upper-Bounds}

All algorithms in this section take as input two datasets $S_{\mu_0}\sim \mu_0^{n_0}$, $S_{\mu_1}\sim \mu_1^{n}$. The corresponding empirical risks are denoted as $\hat{R}_{\mu_0}$, $\hat{R}_{\mu_1}$. Additionally, in the following, whenever $\mu_0$ is known set $\hat{\mu}_0=\mu_0$ and $\hat{R}_{\mu_0}=R_{\mu_0}$. Furthermore, the error rates are expressed in terms of $\epsilon_n=\frac{d_{\mathcal{H}}\log{2n}+\log{\frac{8}{\delta}}}{n}$.

The first algorithm below serves to establish the minimax upper bounds of Theorem \ref{thm1} and \ref{thm2} for the case where $\mathcal{H}$ separates three points.

\begin{alg}\label{algorithm:1:separating-three-points}
Let $\tilde{\mathcal{H}}=\mathcal{H}_{\alpha}(\mu_0)$ when $\mu_0$ is known, otherwise $\tilde{\mathcal{H}}=\mathcal{H}_{\alpha+\epsilon_0/2}(\hat{\mu}_0)$. Define
\begin{align}\label{Algorithm_1}
\hat{h}=\argmin_{} \braces{\hat{R}_{\mu_1}(h): h\in \tilde{\mathcal{H}}} .
\end{align}
\end{alg}
\begin{lemma}[$\mathcal{H}$ separates three points]\label{upp-sep_1}
Fix any $0<\alpha<\frac{1}{2}$ and $\delta>0$. Let $\delta_0,\epsilon_0>0$ such that $\alpha+\epsilon_0<\frac{1}{2}$ when $\mu_0$ is unknown, otherwise $\epsilon_0=\delta_0=0$. Moreover, let $\cal H$ be a class with VC dimension $d_{\cal H}$ that separates three points and $(\mu_0,\mu_1)$ be a pair of distributions such that the set $\mathcal{H}_{\mu_0}(\alpha)$ is non-empty. Suppose that there are $n_0=n_0(\frac{\epsilon_0}{2},\delta_0)$, according to Definition \ref{sample}, and $n$ i.i.d. samples from $\mu_0$ and $\mu_1$, denoted by $S_{\mu_0}$ and $S_{\mu_1}$, respectively. Let $\hat{h}$ be the hypothesis returned by Algorithm \ref{algorithm:1:separating-three-points}. Then, with probability at least $1-\delta_0$, $\hat{h}\in \mathcal{H}_{\alpha+\epsilon_0}(\mu_0)$. Furthermore, we have
\begin{align}\label{upper_bound_separate_lemma}
\underset{S_{\mu_0},S_{\mu_1}}{\mathbb{P}}\bigg(\mathcal{E}_{\alpha}(\hat{h})\cdot \indicator{\hat{h}\in \mathcal{H}_{\alpha+\epsilon_0}(\mu_0)}\leq 4\sqrt{\epsilon}_n\bigg)
\geq 1-\delta_0-\delta.
\end{align}

\end{lemma}
Next, we consider the following procedure for the case where $\mathcal{H}$ does not separate three points.
\begin{alg}\label{algorithm:not-separating-three-points} Let $\tilde{\mathcal{H}}=\mathcal{H}_{\alpha}(\mu_0)$ when $\mu_0$ is known, otherwise $\tilde{\mathcal{H}}=\mathcal{H}_{\alpha+\epsilon_0/2}(\hat{\mu}_0)$. Define $\hat{h}$ as the output of the following procedure:
\begin{align*}
        &\text{If}\ \tilde{\mathcal{H}}\ \text{admits a maximal element, then let} \ \hat{h} \ \text{be the maximal element.}\nonumber\\
        & \text{Otherwise, let } \hat{h}=\argmin_{} \braces{\hat{R}_{\mu_1}(h): h\in \tilde{\mathcal{H}}}.
    \end{align*}
 \end{alg}
\begin{lemma}[$\mathcal{H}$ does not separate three points]\label{upp-not_sep_1}
Fix any $0<\alpha<\frac{1}{2}$ and $\delta>0$. Let $\delta_0,\epsilon_0>0$ such that $\alpha+\epsilon_0<\frac{1}{2}$ when $\mu_0$ is unknown, otherwise $\epsilon_0=\delta_0=0$. Moreover, let $\cal H$ be a hypothesis class with VC dimension $d_{\cal H}$ that separates three points and $(\mu_0,\mu_1)$ be a pair of distributions such that the set $\mathcal{H}_{\alpha}(\mu_0)$ is non-empty. Suppose that there are $n_0=n_0(\frac{\epsilon_0}{2},\delta_0)$, according to Definition \ref{sample}, and $n$ i.i.d. samples from $\mu_0$ and $\mu_1$, respectively. Let $\hat{h}$ be the hypothesis returned by Algorithm \ref{algorithm:not-separating-three-points}. 
\begin{enumerate}
    \item[(a)] With probability at least $1-\delta_0$, we have $\hat{h}\in \mathcal{H}_{\alpha+\epsilon_0}(\mu_0)$. Furthermore, we have for some universal constant $C$
\begin{align}\label{upper_notseparate_lemma_4}
\underset{S_{\mu_0},S_{\mu_1}}{\mathbb{P}}\bigg(\mathcal{E}_{\alpha}(\hat{h})\cdot\indicator{\hat{h}\in\mathcal{H}_{\alpha+\epsilon_0}(\mu_0)}\leq  C\cdot \epsilon_n\bigg)
\geq 1-\delta_0-\delta.
\end{align}
\item[(b)] Suppose that for all distributions $\mu_0$ over  $(\mathcal{X},\Sigma)$ with finite support, $\mathcal{H}_{\alpha+\epsilon_0}(\mu_0)$ admits a maximal
element. Then with probability at least $1-\delta_0$, $\hat{h}\in \mathcal{H}_{\alpha+\epsilon_0}(\mu_0)$ and $\mathcal{E}_{\alpha}(\hat{h})=0$.
\end{enumerate}
\end{lemma}
\subsection{Supporting Lower-Bounds}\label{sec: Minimax Lower-Bounds}
\begin{figure}[t]
  \centering
\includegraphics[height=0.32\textwidth,width=0.55\textwidth]{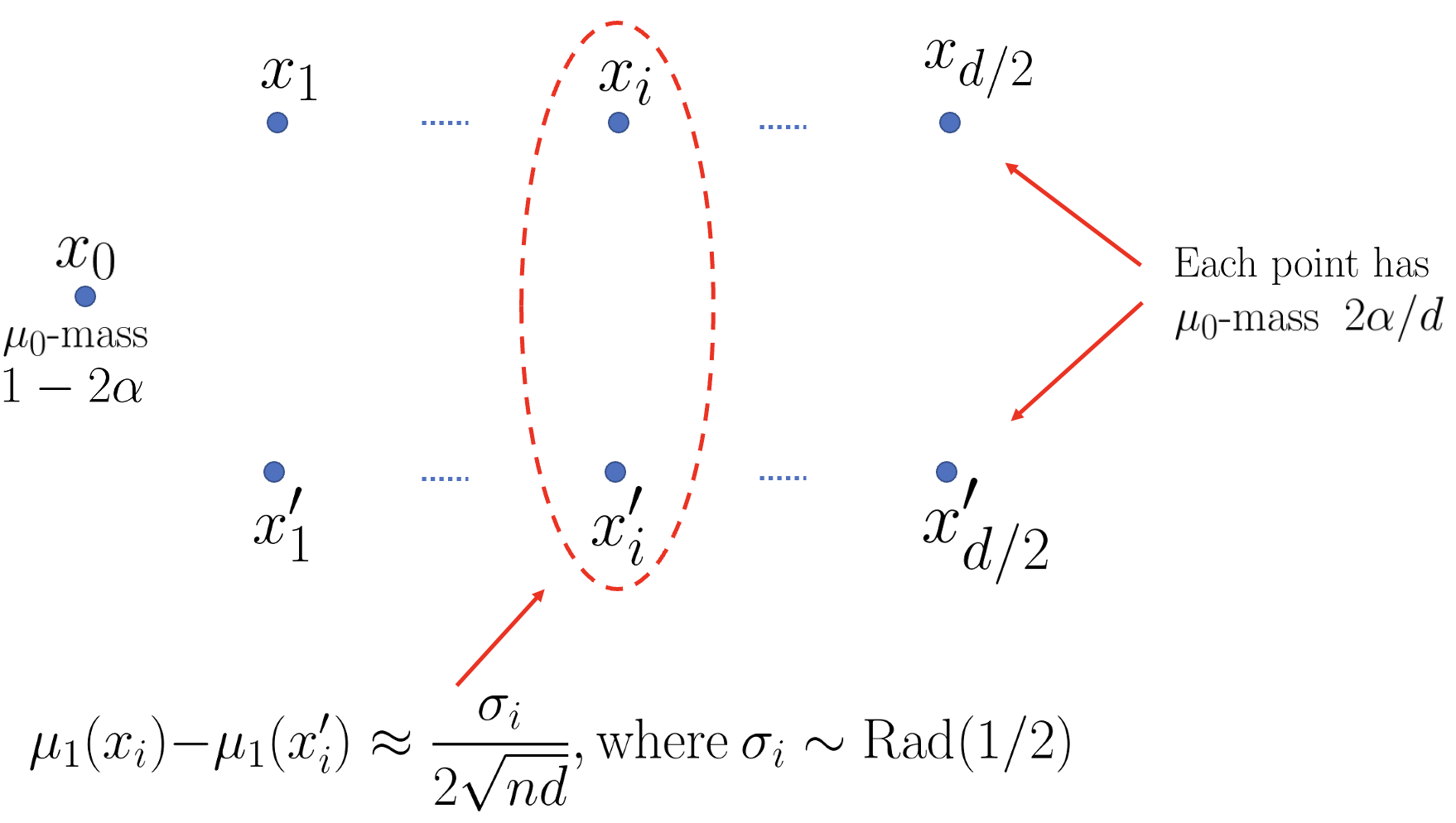}
  \caption{Illustration of $\Omega (1/\sqrt{n})$ lower-bound construction in Lemma \ref{4.3}. Letting $d\approx d_\Hyp$, each $h\in \Hyp_\alpha(\mu_0)$ picks (i.e, is 1 on) at most $d/2$ points out of $\braces{x_i, x_i'}_{i=1}^{d/2}$. 
  The points $(x_i, x_i')$ are paired so that, effectively, the learner's choice of $h\in \Hyp_{\alpha}(\mu_0)$ reduces to deciding for each $i$, which of $x_i$ or $x_i'$ has the highest $\mu_1$-mass (randomized in $\sigma_i$). \emph{Note that this construction requires three-points-separation:} $\exists h,, h'$'s in $\Hyp_\alpha(\mu_0)$, both $0$ on $x_0$, but differing on $x_i, x_i'$ for some $i$. 
  }\label{FIG_VC>16}
\end{figure}

To address both Theorems \ref{thm1} and \ref{thm2} at once in this section, we consider a more powerful variant of learners: these have knowledge of $\mu_0$ in the sense that we lower-bound error over a subfamily ${\cal U}(\mu_0)$ with fixed $\mu_0$, and also are allowed to return a hypothesis from $\Hyp_{\alpha} (\mu_0)$ with probability at least $1- \delta_0$ (according to internal randomness, e.g., by sampling from $\mu_0$). We refer to such $\hat h$ as a {\bf randomized $(\alpha, \delta_0)$-learner}.

\begin{lemma}[$\mathcal{H}$ separates three points]\label{4.3}
    Let $\mathcal{H}$ be a class with VC dimension $d_\mathcal{H}$ that separates three points. Fix $0<\alpha<\frac{1}{2}, 0<\epsilon_0\leq \frac{\alpha}{d_{\mathcal{H}}}$, and $0\leq\delta_0<\frac{1}{2}$. Let $\hat{h}$ denote any randomized $(\alpha, \delta_0)$-learner, as defined above, with access to $S_{\mu_1}\sim \mu_1^{n}$. Assume $n$ is large enough so that $\frac{d_{\mathcal{H}}}{n}\leq 1$. Then $\exists \mu_0$, with $\Hyp_{\alpha}(\mu_0) = \Hyp_{\alpha + \epsilon_0}(\mu_0)$ such that, for any such $\hat h$, for some universal constants $c,c'$: 
\begin{align}\label{low_ineq}
\sup_{\mathcal{U}(\mu_0)} \mathbb{P}\bigg(\mathcal{E}_{\alpha}(\hat{h})\cdot \indicator{\hat{h}\in \mathcal{H}_{\alpha+\epsilon_0}(\mu_0)}>c\cdot\sqrt{\frac{d_{\mathcal{H}}}{n}}
\bigg)\geq c'.
\end{align}
where the probability is taken w.r.t. $S_{\mu_1}$ and the randomness in the algorithm.
\end{lemma}
Figure \ref{FIG_VC>16} illustrates the construction of the distributions $\mu_0,\mu_1$ used to establish the lower-bound of $\Omega (1/\sqrt{n})$ in Lemma \ref{4.3}. We pick order $d_{\mathcal{H}}$ points from $\mathcal{X}$ shattered by $\mathcal{H}$ and construct a distribution $\mu_0$ and a family of distributions $\mu_1$ on these points. The construction essentially randomizes the $\mu_1$ masses of sets $\braces{h =1}$, $h \in \Hyp_{\mu_0}(\alpha)$ and the learner has to figure out which has largest mass. 

\begin{figure}
\centering
\begin{subfigure}{.46\textwidth}
  \centering
  \includegraphics[height=.6\textwidth,width=.9\textwidth]{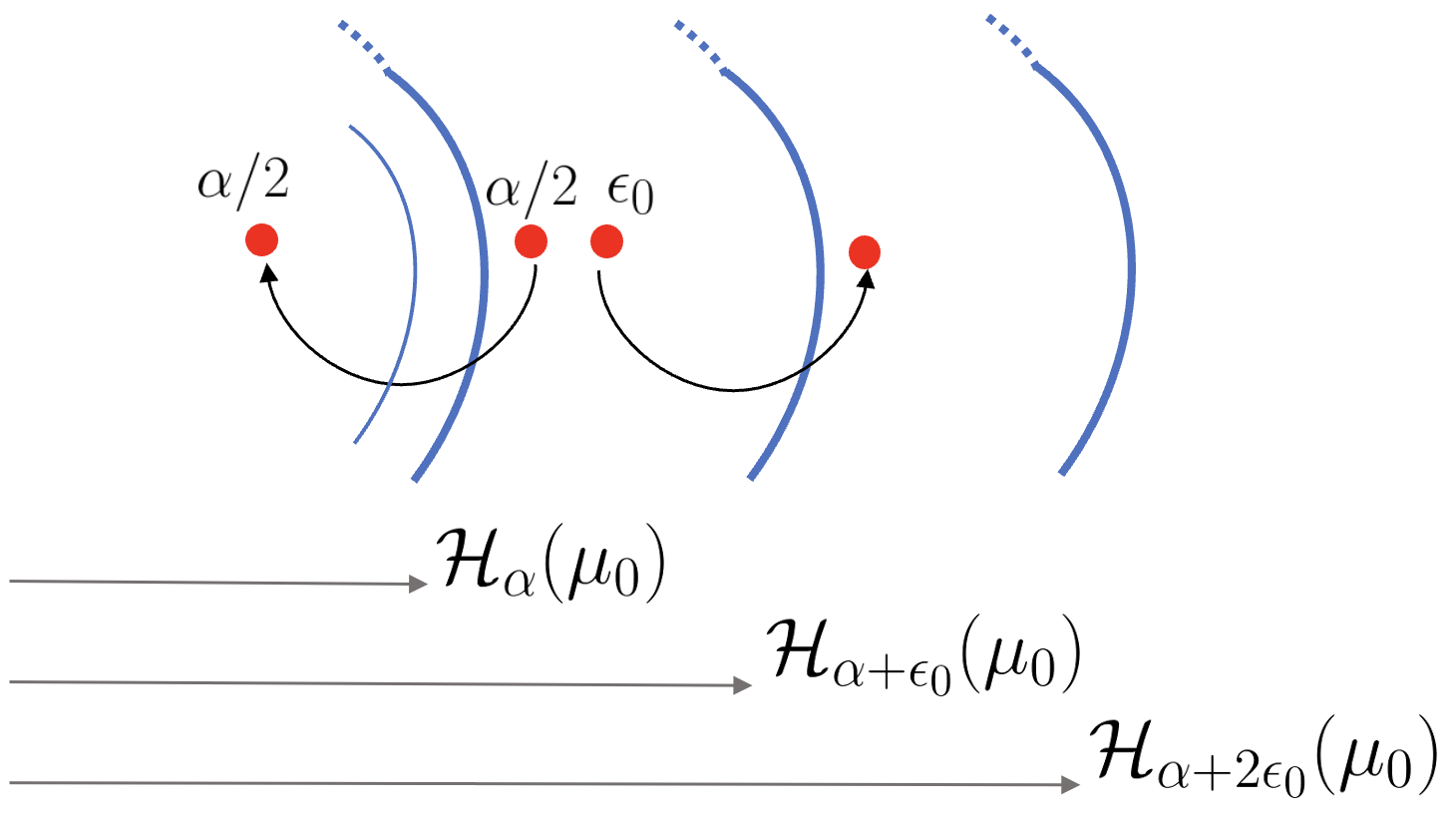}
  \caption{}
  \label{fig3:1}
\end{subfigure}
\hfill
\hspace{-60mm}
\begin{subfigure}{.46\textwidth}
  \centering
\includegraphics[height=.6\textwidth,width=.8\textwidth]{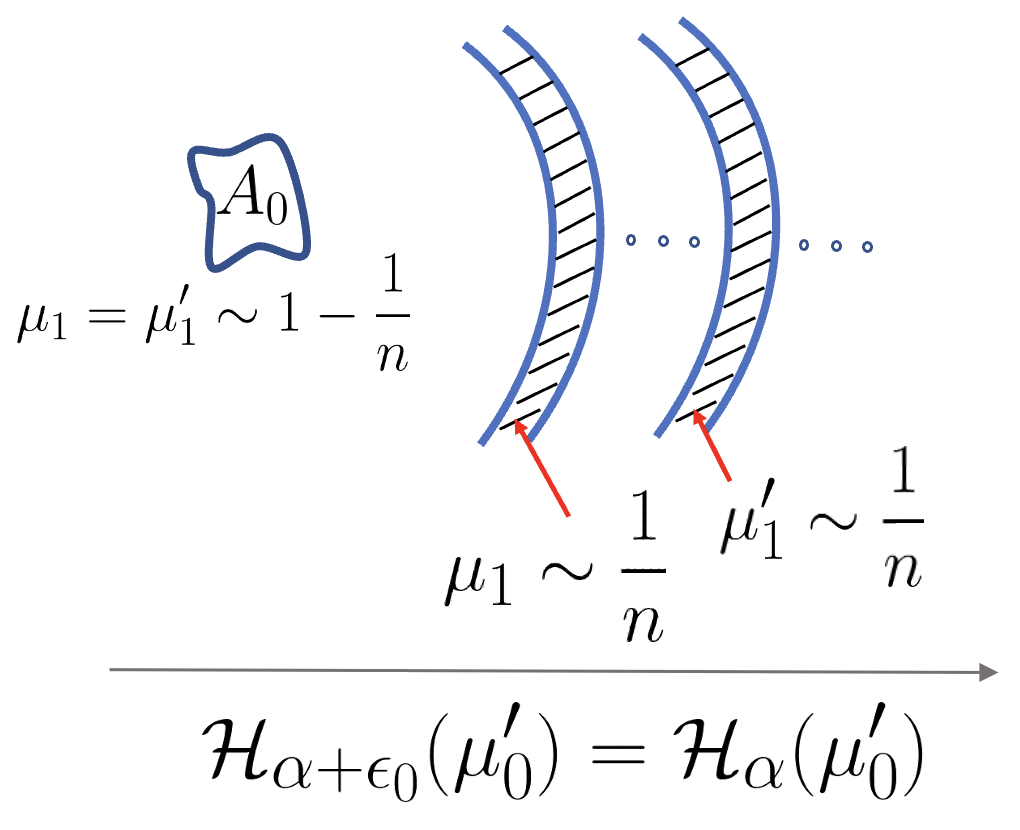}
  \caption{} 
  \label{fig3:2}
\end{subfigure}
\caption{Lower-bound constructions for cases (ii.a) of Theorems 
\ref{thm1} and \ref{thm2}. Theorem \ref{thm2} (ii.a), starts with a reduction: for any $\Hyp_{\alpha+\epsilon_0} (\mu_0)$ with no maximal element, we construct a new $\mu_0'$ such that $\Hyp_{\alpha+\epsilon_0} (\mu_0') = \Hyp_{\alpha} (\mu'_0)$ and also has no maximal element. Subfigure (a): the main idea is to move $\mu_0$-mass out of regions $\braces{h=1}$, $h \in \Hyp_{\alpha+\epsilon_0} (\mu_0) \setminus \Hyp_{\alpha} (\mu_0)$ carefully to ensure that $\Hyp_{\alpha+\epsilon_0} (\mu_0') = \Hyp_{\alpha+\epsilon_0} (\mu_0)$, i.e., without reducing the mass of regions outside $\Hyp_{\alpha + \epsilon_0}(\mu_0)$. These involves technical corner cases handled in Lemma \ref{lemma4.8}. Subfigure (b): A family of distributions $\mu_1$'s can then be defined over $\Hyp_{\alpha}(\mu_0')$, that all put the bulk of their mass on a set $A_0 \doteq \braces{h_0=1}$ for some $h_0 \in \Hyp_{\alpha}(\mu_0')$, but differ in where the put the remaining mass of order $n^{-1}$. The learner has to identify where the remaining mass resides. }
\label{fig3}
\end{figure}

\begin{lemma}[$\mathcal{H}$ does not separate three points]\label{lower}
    Let $\mathcal{H}$ be a class with VC dimension $d_{\mathcal{H}}$ that does not separate three points. Fix $0<\alpha<\frac{1}{2}$ and $\delta_0\geq 0$. Let $\hat{h}$ denote any randomized $(\alpha,\delta_0)$-learner with access to $S_{\mu_1}\sim \mu_1^{n}$. Assume that $n\geq 2$. Consider any $\mu_0$ such that $\mathcal{H}_{\alpha}(\mu_0)$ does not admit a maximal element. Then for any $\hat{h}$, there exist universal constants $c,c'$ such that
\begin{align}\label{low_ineq_lemma6}
\sup_{\mathcal{U}(\mu_0)}\mathbb{P}\bigg(\mathcal{E}_{\alpha}(\hat{h})>c\cdot\frac{d_{\mathcal{H}}}{n}
\bigg)\geq c'.
\end{align}
\end{lemma}

The construction for this lemma is illustrated in Figure \ref{fig3} (b). The lemma immediately applies to Theorem \ref{thm1} (ii.a), while Theorem \ref{thm2} (ii.a) further requires constructing a suitable $\mu_0'$ (Figure \ref{fig3} (a)). 

\subsection{Proof of Theorem \ref{thm1}: Exact Level $\alpha$}\label{sec: proof:theorem 1}
  (i) {\bf $\mathcal{H}$ separates three points.}\\
$\bullet$ \emph{Upper bound.} By Lemma \ref{upp-sep_1}, for the hypothesis $\hat{h}$ returned by Algorithm \ref{algorithm:1:separating-three-points} we get  $\mathcal{E}_{\alpha}(\hat{h})\leq 4\sqrt{\frac{d_{\mathcal{H}}\log{2n}+\log{\frac{8}{\delta}}}{n}}$ with probability at least $1-\delta$. Therefore, if we choose $\delta=\frac{1}{n}$, for some numerical constants $C_1,C_2$ we obtain
\begin{align*}
    \Ex_{S_{\mu_1}} [\mathcal{E}_{\alpha}(\hat{h})]\leq C_1\sqrt{\frac{d_{\mathcal{H}}}{n}\log{n}}+\delta
    \leq C_2\sqrt{\frac{d_{\mathcal{H}}}{n}\log{n}}.
\end{align*}

$\bullet$ \emph{Lower bound.} Lemma \ref{4.3} provides a minimax lower bound for any randomized $(\alpha,\delta_0)$-learner with access to $S_{\mu_1}\sim \mu_1^n$. This implies that \eqref{low_ineq} remains valid if we consider exact $\alpha$-learners, according to Definition \ref{Distribution-free-learner}, as an exact $\alpha$-learner is also a randomized $(\alpha,\delta_0)$-learner. By using Markov's inequality we can deduce the minimax lower bound in Theorem \ref{thm1} (i) from \eqref{low_ineq}.

\noindent {\bf (ii) $\mathcal{H}$ does not separate three points.} \\
  $\bullet$ \emph{Upper bound.} By Lemma \ref{upp-not_sep_1}, 
  if $\mathcal{H}_{\alpha}(\mu_0)$ does not admit a maximal element, then the hypothesis $\hat{h}$ returned by Algorithm \ref{algorithm:not-separating-three-points}
  with probability at least $1-\delta$ satisfies $\mathcal{E}_{\alpha}(\hat{h})\leq C \frac{d_{\mathcal{H}}}{n}\log{n}$ for a numerical constant $C$. Therefore, if we choose $\delta=\frac{1}{n}$ we obtain
\begin{align*}
    \Ex_{S_{\mu_1}} [\mathcal{E}_{\alpha}(\hat{h})]\leq C\frac{d_{\mathcal{H}}}{n}\log{n}+\delta\leq C\frac{d_{\mathcal{H}}}{n}\log{n}.
\end{align*}
Furthermore for Theorem \ref{thm1} (ii.b), if $\mathcal{H}_{\alpha}(\mu_0)$ admits a maximal element, then by Lemma \ref{upp-not_sep_1} we get that $\hat{h}\in \mathcal{H}_{\alpha}(\mu_0)$ and $\mathcal{E}_{\alpha}(\hat{h})= 0$.\\
$\bullet$ \emph{Lower bound.}
  Direct application of
Lemma \ref{lower}  implies the lower bound in Theorem \ref{thm1} (ii.a).
\begin{flushright}{$\square$}
\end{flushright}

\subsection{Proof of Theorem \ref{thm2}: Approximate Level $\alpha$}\label{sec: Proof: theorem 2}
For part (ii.a) of Theorem \ref{thm2} we require the following additional lemma.
\begin{lemma}\label{lemma4.8}
    Suppose that there exists a distribution $\mu_0$ over the measurable space $(\mathcal{X},\Sigma)$ with finite support such that $\mathcal{H}_{\alpha+\epsilon_0}(\mu_0)$ does not
admit a maximal element. Then, there also exists a measure $\tilde{\mu}_0$ over the measurable space $(\mathcal{X},\Sigma)$ with finite support such that $\mathcal{H}_{\alpha}(\mu'_0)=\mathcal{H}_{\alpha+\epsilon_0}(\mu'_0)$ which does not admit a maximal element.
\end{lemma}
{\bf(i) $\mathcal{H}$ separates three points.}\\
$\bullet$ \emph{Upper bound.} By Lemma \ref{upp-sep_1}, for the hypothesis $\hat{h}$ returned by Algorithm \ref{algorithm:1:separating-three-points}, first we get that $\hat{h}\in \mathcal{H}_{\alpha+\epsilon_0}(\mu_0)$ with probability at least $1-\delta_0$. Furthermore, if we choose $\delta=\frac{1}{n}$, \eqref{upper_bound_separate_lemma} implies that for some numerical constants $C_1,C_2$ we get
\begin{align*}
\Ex_{S_{\mu_0},S_{\mu_1}}\big[\mathcal{E}_{\alpha}(\hat{h})\cdot\mathbb{1}\big\{\hat{h}\in \mathcal{H}_{\alpha+\epsilon_0}(\mu_0)\big\}\big]\leq C_1\sqrt{\frac{d_{\mathcal{H}}}{n}\log{n}}
    +(\delta_0+\delta)\leq C_2\sqrt{\frac{d_{\mathcal{H}}}{n}\log{n}} +\delta_0,
\end{align*}
      $\bullet$ \emph{Lower bound.} Lemma \ref{4.3} implies that there exists a distribution $\mu_0$ with $\mathcal{H}_{\alpha}(\mu_0)=\mathcal{H}_{\alpha+\epsilon_0}(\mu_0)$ such that for any randomized $(\alpha,\delta_0)$-learner, \eqref{low_ineq} holds. Furthermore, when $\mathcal{H}_{\alpha}(\mu_0)=\mathcal{H}_{\alpha+\epsilon_0}(\mu_0)$, an $(\epsilon_0,\delta_0)$-approximate $\alpha$-learner is also a randomized $(\alpha,\delta_0)$-learner. Therefore, using Markov's inequality implies the lower bound in Theorem \ref{thm2} (i) from \eqref{low_ineq}.\\
{\bf(ii) $\mathcal{H}$ does not separate three points.}\\
$\bullet$ \emph{Upper bound.} By Lemma \ref{upp-not_sep_1}, for the hypothesis $\hat{h}$ returned by Algorithm \ref{algorithm:not-separating-three-points} we get that $\hat{h}\in \mathcal{H}_{\alpha+\epsilon_0}(\mu_0)$ with probability at least $1-\delta_0$. Furthermore, by choosing $\delta=\frac{1}{n}$ we obtain
      \begin{align*}
\Ex_{S_{\mu_0},S_{\mu_1}}[\mathcal{E}_{\alpha}(\hat{h})\cdot \indicator{\hat{h}\in \mathcal{H}_{\alpha+\epsilon_0}(\mu_0)}]\leq C_1\frac{d_{\mathcal{H}}}{n}\log{n}+(\delta_0+\delta)\leq C_2\frac{d_{\mathcal{H}}}{n}\log{n}+\delta_0,
      \end{align*}
where $C_1,C_2$ are numerical constants. Moreover, for Theorem \ref{thm2} (ii.b), if for all distributions $\mu_0$ over the measurable space $(\mathcal{X},\Sigma)$ with finite support, $\mathcal{H}_{\alpha+\epsilon_0}(\mu_0)$ admits a maximal element, then with probability at least $1-\delta_0$, we have $\hat{h}\in \mathcal{H}_{\alpha+\epsilon_0}(\mu_0)$ and $\mathcal{E}_{\alpha}(\hat{h})=0$, which implies \begin{align*}
        \inf_{\hat h}\ \mathbb{P}\bigg(\mathcal{E}_{\alpha}(\hat h)\cdot\mathbb{1}\big\{\hat{h}\in \mathcal{H}_{\alpha+\epsilon_0}(\mu_0)\big\}=0 \bigg)\ \geq\ 1-\delta_0.
     \end{align*}\\
      $\bullet$ \emph{Lower bound.} We aim to call on Lemma \ref{lower} which addresses randomized $(\alpha,\delta_0)$-learners which return a hypothesis from $\mathcal{H}_{\alpha}(\mu_0)$ rather than from $\mathcal{H}_{\alpha + \epsilon_0}(\mu_0)$. We therefore combine with Lemma \ref{lemma4.8}, to reveal a distribution $\mu'_0$ such that $\mathcal{H}_{\alpha}(\mu'_0)=\mathcal{H}_{\alpha+\epsilon_0}(\mu'_0)$ and satisfying the conditions of Lemma \ref{lower}. This establishes the lower-bound in Theorem \ref{thm2} (ii.a). 
     \begin{flushright}{$\square$}
\end{flushright}
\bibliographystyle{abbrvnat}
\bibliography{references}

\begin{thebibliography}{32}
\providecommand{\natexlab}[1]{#1}
\providecommand{\url}[1]{\texttt{#1}}
\expandafter\ifx\csname urlstyle\endcsname\relax
  \providecommand{\doi}[1]{doi: #1}\else
  \providecommand{\doi}{doi: \begingroup \urlstyle{rm}\Url}\fi

\bibitem[Bartlett and Mendelson(2006)]{bartlett2006empirical}
P.~L. Bartlett and S.~Mendelson.
\newblock Empirical minimization.
\newblock \emph{Probability theory and related fields}, 135\penalty0
  (3):\penalty0 311--334, 2006.

\bibitem[Blanchard et~al.(2010)Blanchard, Lee, and Scott]{blanchard2010semi}
G.~Blanchard, G.~Lee, and C.~Scott.
\newblock Semi-supervised novelty detection.
\newblock \emph{The Journal of Machine Learning Research}, 11:\penalty0
  2973--3009, 2010.

\bibitem[Bourzac(2014)]{bourzac2014diagnosis}
K.~Bourzac.
\newblock Diagnosis: early warning system.
\newblock \emph{Nature}, 513\penalty0 (7517):\penalty0 S4--S6, 2014.

\bibitem[Cannon et~al.(2002)Cannon, Howse, Hush, and
  Scovel]{cannon2002learning}
A.~Cannon, J.~Howse, D.~Hush, and C.~Scovel.
\newblock Learning with the neyman-pearson and min-max criteria.
\newblock \emph{Los Alamos National Laboratory, Tech. Rep. LA-UR}, pages
  02--2951, 2002.

\bibitem[Casasent and Chen(2003)]{casasent2003radial}
D.~Casasent and X.-w. Chen.
\newblock Radial basis function neural networks for nonlinear fisher
  discrimination and neyman--pearson classification.
\newblock \emph{Neural networks}, 16\penalty0 (5-6):\penalty0 529--535, 2003.

\bibitem[Fan et~al.(2023)Fan, Tong, Wu, and Yao]{fan2023neyman}
J.~Fan, X.~Tong, Y.~Wu, and S.~Yao.
\newblock Neyman-pearson and equal opportunity: when efficiency meets fairness
  in classification.
\newblock \emph{arXiv preprint arXiv:2310.01009}, 2023.

\bibitem[Han et~al.(2008)Han, Chen, and Sun]{han2008analysis}
M.~Han, D.~Chen, and Z.~Sun.
\newblock Analysis to neyman-pearson classification with convex loss function.
\newblock \emph{Analysis in Theory and Applications}, 24:\penalty0 18--28,
  2008.

\bibitem[Hanneke(2016)]{hanneke2016optimal}
S.~Hanneke.
\newblock The optimal sample complexity of pac learning.
\newblock \emph{The Journal of Machine Learning Research}, 17\penalty0
  (1):\penalty0 1319--1333, 2016.

\bibitem[Jose et~al.(2018)Jose, Malathi, Reddy, and Jayaseeli]{jose2018survey}
S.~Jose, D.~Malathi, B.~Reddy, and D.~Jayaseeli.
\newblock A survey on anomaly based host intrusion detection system.
\newblock In \emph{Journal of Physics: Conference Series}, volume 1000, page
  012049. IOP Publishing, 2018.

\bibitem[Koltchinskii(2006)]{koltchinskii2006local}
V.~Koltchinskii.
\newblock Local rademacher complexities and oracle inequalities in risk
  minimization.
\newblock 2006.

\bibitem[Kumar and Lim(2019)]{kumar2019edima}
A.~Kumar and T.~J. Lim.
\newblock Edima: Early detection of iot malware network activity using machine
  learning techniques.
\newblock In \emph{2019 IEEE 5th World Forum on Internet of Things (WF-IoT)},
  pages 289--294. IEEE, 2019.

\bibitem[Lehmann et~al.(1986)Lehmann, Romano, and Casella]{lehmann1986testing}
E.~L. Lehmann, J.~P. Romano, and G.~Casella.
\newblock \emph{Testing statistical hypotheses}, volume~3.
\newblock Springer, 1986.

\bibitem[Lin and Deng(2023)]{lin2023gbm}
Z.~Lin and Q.~Deng.
\newblock Gbm-based bregman proximal algorithms for constrained learning.
\newblock \emph{arXiv preprint arXiv:2308.10767}, 2023.

\bibitem[Ma et~al.(2020)Ma, Lin, and Yang]{ma2020quadratically}
R.~Ma, Q.~Lin, and T.~Yang.
\newblock Quadratically regularized subgradient methods for weakly convex
  optimization with weakly convex constraints.
\newblock In \emph{International Conference on Machine Learning}, pages
  6554--6564. PMLR, 2020.

\bibitem[Mammen and Tsybakov(1999)]{mammen1999smooth}
E.~Mammen and A.~B. Tsybakov.
\newblock Smooth discrimination analysis.
\newblock \emph{The Annals of Statistics}, 27\penalty0 (6):\penalty0
  1808--1829, 1999.

\bibitem[Massart and N{\'e}d{\'e}lec(2006)]{massart2006risk}
P.~Massart and {\'E}.~N{\'e}d{\'e}lec.
\newblock Risk bounds for statistical learning.
\newblock 2006.

\bibitem[Mohri et~al.(2018)Mohri, Rostamizadeh, and
  Talwalkar]{mohri2018foundations}
M.~Mohri, A.~Rostamizadeh, and A.~Talwalkar.
\newblock \emph{Foundations of machine learning}.
\newblock MIT press, 2018.

\bibitem[Rigollet and Tong(2011)]{rigollet2011neyman}
P.~Rigollet and X.~Tong.
\newblock Neyman-pearson classification, convexity and stochastic constraints.
\newblock \emph{Journal of Machine Learning Research}, 2011.

\bibitem[Scott(2007)]{scott2007performance}
C.~Scott.
\newblock Performance measures for neyman--pearson classification.
\newblock \emph{IEEE Transactions on Information Theory}, 53\penalty0
  (8):\penalty0 2852--2863, 2007.

\bibitem[Scott(2019)]{scott2019generalized}
C.~Scott.
\newblock A generalized neyman-pearson criterion for optimal domain adaptation.
\newblock In \emph{Algorithmic Learning Theory}, pages 738--761. PMLR, 2019.

\bibitem[Scott and Nowak(2005)]{scott2005neyman}
C.~Scott and R.~Nowak.
\newblock A neyman-pearson approach to statistical learning.
\newblock \emph{IEEE Transactions on Information Theory}, 51\penalty0
  (11):\penalty0 3806--3819, 2005.

\bibitem[Shalev-Shwartz and Ben-David(2014)]{shalev2014understanding}
S.~Shalev-Shwartz and S.~Ben-David.
\newblock \emph{Understanding machine learning: From theory to algorithms}.
\newblock Cambridge university press, 2014.

\bibitem[Tian and Feng(2021)]{tian2021neyman}
Y.~Tian and Y.~Feng.
\newblock Neyman-pearson multi-class classification via cost-sensitive
  learning.
\newblock \emph{arXiv preprint arXiv:2111.04597}, 2021.

\bibitem[Tong(2013)]{tong2013plug}
X.~Tong.
\newblock A plug-in approach to neyman-pearson classification.
\newblock \emph{The Journal of Machine Learning Research}, 14\penalty0
  (1):\penalty0 3011--3040, 2013.

\bibitem[Tong et~al.(2018)Tong, Feng, and Li]{tong2018neyman}
X.~Tong, Y.~Feng, and J.~J. Li.
\newblock Neyman-pearson classification algorithms and np receiver operating
  characteristics.
\newblock \emph{Science advances}, 4\penalty0 (2):\penalty0 eaao1659, 2018.

\bibitem[Tong et~al.(2020)Tong, Xia, Wang, and Feng]{tong2020neyman}
X.~Tong, L.~Xia, J.~Wang, and Y.~Feng.
\newblock Neyman-pearson classification: parametrics and sample size
  requirement.
\newblock \emph{The Journal of Machine Learning Research}, 21\penalty0
  (1):\penalty0 380--427, 2020.

\bibitem[Tsybakov(2009)]{Tsybakov:1315296}
A.~B. Tsybakov.
\newblock \emph{{Introduction to Nonparametric Estimation}}.
\newblock Springer series in statistics. Springer, Dordrecht, 2009.
\newblock \doi{10.1007/b13794}.

\bibitem[Vapnik and Chervonenkis(2015)]{vapnik2015uniform}
V.~N. Vapnik and A.~Y. Chervonenkis.
\newblock On the uniform convergence of relative frequencies of events to their
  probabilities.
\newblock In \emph{Measures of complexity: festschrift for alexey
  chervonenkis}, pages 11--30. Springer, 2015.

\bibitem[Wainwright(2019)]{wainwright2019high}
M.~J. Wainwright.
\newblock \emph{High-dimensional statistics: A non-asymptotic viewpoint},
  volume~48.
\newblock Cambridge university press, 2019.

\bibitem[Zeng et~al.(2022)Zeng, Dobriban, and Cheng]{zeng2022bayes}
X.~Zeng, E.~Dobriban, and G.~Cheng.
\newblock Bayes-optimal classifiers under group fairness.
\newblock \emph{arXiv preprint arXiv:2202.09724}, 2022.

\bibitem[Zhao et~al.(2016)Zhao, Feng, Wang, and Tong]{zhao2016neyman}
A.~Zhao, Y.~Feng, L.~Wang, and X.~Tong.
\newblock Neyman-pearson classification under high-dimensional settings.
\newblock \emph{The Journal of Machine Learning Research}, 17\penalty0
  (1):\penalty0 7469--7507, 2016.

\bibitem[Zheng et~al.(2011)Zheng, Loziczonek, Georgescu, Zhou, Vega-Higuera,
  and Comaniciu]{zheng2011machine}
Y.~Zheng, M.~Loziczonek, B.~Georgescu, S.~K. Zhou, F.~Vega-Higuera, and
  D.~Comaniciu.
\newblock Machine learning based vesselness measurement for coronary artery
  segmentation in cardiac ct volumes.
\newblock In \emph{Medical Imaging 2011: Image Processing}, volume 7962, pages
  489--500. Spie, 2011.

\end{thebibliography}
\newpage

\appendix
\section{Proofs}
\subsection{Supporting Results}
We will make use of the following classical concentration results. Note that, the VC dimension of a collection of sets ${\cal B}\subset \Sigma$, is defined via the corresponding class of indicators over sets in $\cal B$. 
\begin{lemma}[Relative VC bounds \citep{vapnik2015uniform}]\label{relative_vc_lemma}
Let $\mathcal{B}$ be a collection of sets in $\Sigma$ with VC dimension $d_{\mathcal{B}}$ and $0<\delta<1$, and define $\epsilon_n=\frac{d_{\mathcal{H}}\log{2n}+\log{\frac{8}{\delta}}}{n}$. Then with probability at least $1-\delta$ w.r.t. $S_{\mu}\sim \mu^{n}$, for any set $B\in \mathcal{B}$ we have
$$\mu(B)\leq \hat{\mu}(B)+\sqrt{\hat{\mu}(B)\epsilon_n}+\epsilon_n \ \ \ \text{and}\ \ \ \hat{\mu}(B)\leq \mu(B)+\sqrt{\mu(B)\epsilon_n}+\epsilon_n$$

where $\hat{\mu}(B)\doteq \frac{1}{n} \sum_{X\in S_{\mu}}\indicator{X\in B}$.
\end{lemma} 

The following simple lemma establishes that $\Hyp$ must be highly structured if it does not satisfy three-points-separation. 

\begin{lemma}\label{total}
    Suppose that $\mathcal{H}$ does not separate three points, and let $\alpha<\frac{1}{2}$. Then, for any $\mu_0$, either $\mathcal{H}_{\alpha}(\mu_0)$ contains at most one element, or $\mathcal{H}_{\alpha}(\mu_0)$ is totally ordered by inclusion, i.e., for any $h,h'\in \mathcal{H}_{\alpha}(\mu_0)$ one of $\{h=1\}$ or $\{h'=1\}$ contains the other.
\end{lemma}
\begin{proof}
Consider any two distinct hypotheses $h,h'\in\mathcal{H}_{\alpha}(\mu_0) \subset \Hyp$; since $\Hyp$ does not separate two points, we know that either (a) or (b) in Definition \ref{def:three_points} does not hold. However, (a) must hold since: 
$$\mu_0(\braces{h = 1}\cup \braces{h'=1}) \leq 2\alpha < 1,$$
i.e., $\braces{h = 1}\cup \braces{h'=1}\neq \X$. It follows that 
one of $\braces{h = 1}$ and $\braces{h' = 1}$ contains the other. 
\end{proof}
When dealing with totally ordered sets, we will need to establish concentration results on the differences of sets. The following characterizes their VC dimension. 
\begin{lemma}\label{lemma_new_class}
Let $\mathcal{B}$ be a subset of $\Sigma$ with VC dimension $d_{\mathcal{B}}$. Then the collection of differences of sets in $\mathcal{B}$, denoted as  $\tilde{\mathcal{B}}=\{B_1\setminus B_2: B_1,B_2\in \mathcal{B}\}$, has VC dimension of at most $12d_{\mathcal{B}}$.
\end{lemma}
\begin{proof}
    First we show that under the complement operator VC dimension remains the same. This means that VC dimension of $\bar{\mathcal{B}}=\big\{B^C:B\in \mathcal{B}\big\}$ is equal to $d_{\mathcal{B}}$. Let $S$ be a set shattered by $\mathcal{B}$ and $\tilde{S}$ be an arbitrary subset of $S$. Since $S$ is shattered by $\mathcal{B}$, there exists $B\in \mathcal{B}$ such that $B\cap S=S\setminus \tilde{S}$. Then, we have $B^C\cap S=\tilde{S}$, which implies that $S$ is also shattered by $\bar{\mathcal{B}}$. Using the same argument, we can get that every set that is shattered by $\bar{\mathcal{B}}$ is also shattered by $\mathcal{B}$. Hence the VC dimension of $\bar{\mathcal{B}}$ is equal to $d_{\mathcal{B}}$.
    
    Moreover, since $\tilde{\mathcal{B}}=\big \{B_1\cap B_2^C: B_1,B_2\in \mathcal{B}\big\}=\big\{B_1\cap B_2:B_1\in \mathcal{B}, B_2\in \bar{\mathcal{B}}\big\}$, by \cite[][Chapter 3, Exercises 3.15]{mohri2018foundations} we obtain that the VC dimension of $\tilde{\mathcal{B}}$ is at most $12d_{\mathcal{H}}$.
\end{proof}
\subsection{Proofs of Section \ref{sec: Generic Upper-Bounds} (Upper-Bounds)}
\begin{proof}[Proof of Lemma \ref{upp-sep_1} (Upper-Bounds when $\mathcal{H}$ separates three points)]
The event $\{\underset{h\in \mathcal{H}}{\sup}|R_{\mu_0}(h)-\hat{R}_{\mu_0}(h)|\leq \frac{\epsilon_0}{2}\}$, by Definition \ref{sample}, happens with probability
at least $1-\delta_0$. This event immediately implies that $\hat h \in \mathcal{H}_{\alpha+\epsilon_0}(\mu_0)$, and also that $\mathcal{H}_{\alpha}(\mu_0)\subset \mathcal{H}_{\alpha+\epsilon_0/2}(\hat{\mu}_0)$.

To show \eqref{upper_bound_separate_lemma}, now condition on this last event that $\mathcal{H}_{\alpha}(\mu_0)\subset \mathcal{H}_{\alpha+\epsilon_0/2}(\hat{\mu}_0)$ . Now, by Lemma \ref{relative_vc_lemma} we also have that, with probability at least $1-\delta$, 
$\underset{h \in \Hyp}{\sup} |R_{\mu_1}(h)-\hat{R}_{\mu_1}(h)| \leq 2\sqrt{\epsilon}_n$. Under these two events, we have that, for all $h \in \Hyp_\alpha(\mu_0)$, 
\begin{align*}
{R}_{\mu_1}(\hat{h})\leq \hat{R}_{\mu_1}(\hat h) + 2\sqrt{\epsilon_n}\leq \hat{R}_{\mu_1}(h) + 2\sqrt{\epsilon}_n \leq {R}_{\mu_1}(h) + 4\sqrt{\epsilon}_n. 
\end{align*}
In other words, we have with probability at least $1-\delta_0 - \delta$ that 
\begin{align*} 
\mathcal{E}_{\alpha}(\hat{h})\cdot \indicator{\hat{h}\in \mathcal{H}_{\alpha+\epsilon_0}(\mu_0)} 
\leq \mathcal{E}_{\alpha}(\hat{h}) \leq 4\sqrt{\epsilon}_n.
\end{align*}
\end{proof}
\begin{proof}[Proof of Lemma \ref{upp-not_sep_1} (Upper-Bounds when $\mathcal{H}$ does not separate three points)]

\textbf{Part (a):} By Definition \ref{sample}, the event $\{\underset{h\in \mathcal{H}}{\sup}|R_{\mu_0}(h)-\hat{R}_{\mu_0}(h)|\leq \frac{\epsilon_0}{2}\}$ happens with probability
at least $1-\delta_0$, which implies $\hat h \in \mathcal{H}_{\alpha+\epsilon_0}(\mu_0)$ and $\mathcal{H}_{\alpha}(\mu_0)\subset \mathcal{H}_{\alpha+\epsilon_0/2}(\hat{\mu}_0)$. In what follows, we condition on this last event that $\mathcal{H}_{\alpha}(\mu_0)\subset \mathcal{H}_{\alpha+\epsilon_0/2}(\hat{\mu}_0)$.

By Lemma \ref{total}, $\mathcal{H}_{\alpha+\epsilon_0/2}(\hat{\mu}_0)$ is well structured, and we only have to consider three cases. If $\mathcal{H}_{\alpha+\epsilon_0/2}(\hat{\mu}_0)$ contains only a single element, then $\mathcal{H}_{\alpha}(\mu_0)$ must have a single element so that we have $\mathcal{E}_{\alpha}(\hat{h})=0$.

If $\mathcal{H}_{\alpha+\epsilon_0/2}(\hat{\mu}_0)$ contains at least two elements, then it would be totally ordered by Lemma \ref{total}. Moreover, if $\mathcal{H}_{\alpha+\epsilon_0/2}(\hat{\mu}_0)$ admits a maximal element, we get $R_{\mu_1}(\hat{h})\leq R_{\mu_1}(h)$ for any $h\in\mathcal{H}_{\alpha}(\mu_0)$, which implies $\mathcal{E}_{\alpha}(\hat{h})=0$.
Next, consider the case where $\mathcal{H}_{\alpha+\epsilon_0/2}(\hat{\mu}_0)$ does not admit a maximal element. By Lemma \ref{total}, for any $h\in \mathcal{H}_{\alpha}(\mu_0)$, one of $\{\hat{h}=1\}$ or $\{h=1\}$ contains the other. We only need to consider the case $\{\hat{h}=1\}\subset \{h=1\}$ as the other case is trivial. By using Lemma \ref{lemma_new_class} and conditioning on the event in Lemma \ref{relative_vc_lemma} w.r.t. $S_{\mu_1}$ for the class of differences of sets in $\mathcal{H}$, we obtain for a numerical constant $C$
\begin{align*}
    R_{\mu_1}(\hat{h})-R_{\mu_1}(h)=\mu_1(\{h=1\}\setminus\{\hat{h}=1\})\leq C\cdot \epsilon_n
\end{align*}
since $\hat{\mu}_1(\{h=1\}\setminus\{\hat{h}=1\})=\hat{R}_{\mu_1}(\hat{h})-\hat{R}_{\mu_1}(h)=0$. In other words, we have with probability at least $1-\delta_0 - \delta$ that 
\begin{align*} 
\mathcal{E}_{\alpha}(\hat{h})\cdot \indicator{\hat{h}\in \mathcal{H}_{\alpha+\epsilon_0}(\mu_0)} 
\leq \mathcal{E}_{\alpha}(\hat{h}) \leq C\cdot\epsilon_n.
\end{align*}

\textbf{Part (b):} We conisder the case where $\mu_0$ is unkonw as the other case is trivial. By the assumption, we know that $\mathcal{H}_{\alpha+\epsilon_0}(\hat{\mu}_0)$ admits a maximal element. It suffices to show that $\mathcal{H}_{\alpha+\epsilon_0/2}(\hat{\mu}_0)$ also admits a maximal element. By contradiction, suppose that $\mathcal{H}_{\alpha+\epsilon_0/2}(\hat{\mu}_0)$ does not admit a maximal element. Then, we construct a new measure $\mu'_0$ with finite support such that $\mathcal{H}_{\alpha+\epsilon_0/2}(\hat{\mu}_0)=\mathcal{H}_{\alpha+\epsilon_0}(\mu'_0)$, which gives a contradiction as $\mathcal{H}_{\alpha+\epsilon_0}(\mu'_0)$ admits a maximal element by the assumption.

We condition on the event $\mathcal{H}_{\alpha}(\mu_0)\subset \hat{\mathcal{H}}_{\alpha+\epsilon_0/2}(\mu_0)$ and construct $\mu'_0$ as follows. Take a point $z$ such that there exists an $h\in \hat{\mathcal{H}}_{\alpha+\epsilon_0/2}(\mu_0)$ with $z\in \{h=1\}$ (note that if there is no $h\in\mathcal{H}_{\alpha+\epsilon_0/2}(\hat{\mu}_0)$ with $\{h=1\}\neq \emptyset$, then $\mathcal{H}_{\alpha+\epsilon_0/2}(\hat{\mu}_0)$ would have a maximal element). Then define the mass of the point $z$ as $\mu'_{0}(\{z\})=\hat{\mu}_0(\{z\})+\frac{\epsilon_0}{2}$. Moreover, let $h^*$ be the maximal element of $\hat{\mathcal{H}}_{\alpha+\epsilon_0}(\mu_0)$ and define
$$A=\text{supp}(\hat{\mu}_0) \cap \{h^*=1\}^C\ \  \text{and} \ \ \ B=\text{supp}(\hat{\mu}_0) \cap \{h^*=1\}\setminus\{z\}.$$
Note that since  $\hat{\mu}_0(\{h^*=1\}^C)>1-\alpha-\epsilon_0\geq \frac{1}{2}$, we get  $\hat{\mu}_0(A)\geq \frac{1}{2}$. Then define the mass of $\mu'_{0}$ on $A$ and $B$ as $$\mu'_{0}(A)=\hat{\mu}_0(A)-\frac{\epsilon_0}{2}\ \  \text{and}\ \  \mu'_{0}(B)=\hat{\mu}_0(B).$$ Therefore, if $z\in \text{supp}(\hat{\mu}_0)$, then $\text{supp}(\mu'_0)=\text{supp}(\hat{\mu}_0)$, otherwise $\text{supp}(\mu'_0)=\text{supp}(\hat{\mu}_0) \cup \{z\}.$

Next, we show that $\mathcal{H}_{\alpha+\epsilon_0/2}(\hat{\mu}_0)=\mathcal{H}_{\alpha+\epsilon_0}(\mu'_0).$ First, let $h\in\mathcal{H}_{\alpha+\epsilon_0/2}(\hat{\mu}_0)$. Since $\{h=1\}\subset \{h^*=1\}$, we get $$R_{\mu'_0}(h)\leq \hat{R}_{\mu_0}(h)+\frac{\epsilon_0}{2}\leq \alpha+\epsilon_0,$$ which implies that  $\mathcal{H}_{\alpha+\epsilon_0/2}(\hat{\mu}_0)\subset\mathcal{H}_{\alpha+\epsilon_0}(\mu'_0)$. Next, in order to show that $\mathcal{H}_{\alpha+\epsilon_0}(\mu'_0)\subset \mathcal{H}_{\alpha+\epsilon_0/2}(\hat{\mu}_0)$ we show the following statements: (I) $\mathcal{H}_{\alpha+\epsilon_0}(\mu'_0)\subset \mathcal{H}_{\alpha+\epsilon_0}(\hat{\mu}_0)$, and 
    (II) If $h\in \mathcal{H}_{\alpha+\epsilon_0}(\hat{\mu}_0) \setminus\mathcal{H}_{\alpha+\epsilon_0/2}(\hat{\mu}_0)$, then $h \not\in \mathcal{H}_{\alpha+\epsilon_0}(\mu'_0)$.
    
    To show (I), we only need to consider an $h$ such that $h\in \mathcal{H}_{\alpha+\epsilon_0}(\mu'_0)\setminus \mathcal{H}_{\alpha+\epsilon_0/2}(\hat{\mu}_0)$, and then show $h \in \mathcal{H}_{\alpha+\epsilon_0}(\hat{\mu}_0)$. Since $\mathcal{H}_{\alpha+\epsilon_0/2}(\hat{\mu}_0)\subset\mathcal{H}_{\alpha+\epsilon_0}(\mu'_0)$, for all $h'\in \mathcal{H}_{\alpha+\epsilon_0/2}(\hat{\mu}_0)$, we have $\{h'=1\}\subset \{h=1\}$. Therefore, $z\in \{h=1\}$ and $$\hat{R}_{\mu_0}(h)\leq R_{\mu'_0}(h)-\frac{\epsilon_0}{2}+\frac{\epsilon_0}{2}\leq \alpha+\epsilon_0,$$ which implies that $h\in \mathcal{H}_{\alpha+\epsilon_0}(\hat{\mu}_0)$.
    
    To show (II), let $h\in \mathcal{H}_{\alpha+\epsilon_0}(\hat{\mu}_0) \setminus\mathcal{H}_{\alpha+\epsilon_0/2}(\hat{\mu}_0)$. Using the same argument, we get $z\in \{h=1\}$. Furthermore, since $\{h=1\}\subset \{h^*=1\}$, we have $\{h=1\}\cap A=\emptyset$ . Hence,
$$R_{\mu'_0}(h)= \hat{R}_{\mu_0}(h)+\frac{\epsilon_0}{2}>\alpha+\epsilon_0,$$ which implies that $h\not\in \mathcal{H}_{\alpha+\epsilon_0}(\hat{\mu}'_0)$.
\end{proof}
\subsection{Proofs of Section \ref{sec: Minimax Lower-Bounds} (Lower-Bounds)}
\begin{proof}[Proof of Lemma \ref{4.3} (Lower-Bounds when $\mathcal{H}$ separates three points)] The lower bound relies on Tysbakov's method \citep{Tsybakov:1315296}.

\begin{proposition}
\citep{Tsybakov:1315296}\label{tsybakov}
Assume that $M\geq 2$ and the function $\text{dist}(\cdot, \cdot)$ is a semi-metric. Furthermore, suppose that $\{\Pi_{\theta_j}\}_{\theta_j\in \Theta}$ is a family of distributions indexed over a parameter space, $\Theta$,  and $\Theta$ contains elements $\theta_0, \theta_1,..., \theta_M$ such that:
\begin{enumerate}
    \item [(i)] $\text{dist}(\theta_i,\theta_j)\geq 2s>0, \ \ \forall \ 0\leq i<j\leq M$
    
    \item [(ii)] $\Pi_j\ll \Pi_0, \ \ \forall \ j=1,...,M,$ and $\frac{1}{M}\sum_{j=1}^{M} \mathcal{D}_{kl}(\Pi_j|\Pi_0)\leq \gamma \log{M}$ with $0<\gamma<1/8$ and $\Pi_j=\Pi_{\theta_j}$, $j=0,1,...,M$ and $\mathcal{D}_{kl}$ denotes the KL-divergence.
    \end{enumerate}
    Then
    \begin{align*}
\inf_{\hat{\theta}}\sup_{\theta\in \Theta}\Pi_{\theta}&(\text{dist}(\hat{\theta},\theta)\geq s)\geq \frac{\sqrt{M}}{1+\sqrt{M}}\big(1-2\gamma-\sqrt{\frac{2\gamma}{\log{M}}} \big).
    \end{align*}  
\end{proposition}
We also utilize the following proposition for constructing a packing of the parameter space.

\begin{proposition}\label{gilbert}(Gilbert-Varshamov bound) Let $d \geq 8$. Then there exists a subset $\{\sigma_0,...,\sigma_M\}$ of $\{-1,+1\}^d$ such that $\sigma_0=(1,1,...,1)$,
$$\text{dist}(\sigma_j,\sigma_k)\geq \frac{d}{8}, \ \ \forall \ 0\leq j<k\leq M \ \text{and} \ M\geq 2^{d/8},$$
where $\text{dist}(\sigma, \sigma') =\text{card}({i \in [m] : \sigma(i)\neq \sigma'(i)})$ is the Hamming distance.
\end{proposition}
We consider the cases $d_{\mathcal{H}}\geq 17$ and $d_{\mathcal{H}}<17$ separately.
\subsection*{Case I: $\mathbf{d_{\mathcal{H}}\geq 17}$.}
Let $d=d_{\mathcal{H}}-1$ and $d_{\mathcal{H}}$ be odd (If $d_{\mathcal{H}}$ is even then define $d=d_{\mathcal{H}}-2$). Then pick $d_{\mathcal{H}}$ points $\mathcal{S}=\{x_0,x_{1},...,x_{\frac{d}{2}}, x'_{1},..., x'_{\frac{d}{2}}\}$ from $\mathcal{X}$ shattered by $\mathcal{H}$ (if $d_{\mathcal{H}}$ is even, we pick $d_{\mathcal{H}}-1$ points). Moreover, let $\tilde{\mathcal{H}}$ be the projection of $\mathcal{H}$ onto the set $\mathcal{S}$ with the constraint that all $h\in \tilde{\mathcal{H}}$ classify $x_0$ as $0$. Next, we construct a distribution $\mu_0$ and a family of distributions $\{\mu_{1}^{\sigma}\}$ indexed by $\sigma=(\sigma_1,...,\sigma_{\frac{d}{2}})\in \{-1,+1\}^{\frac{d}{2}}$. In the following, we fix $\Delta=c_1\cdot \sqrt{\frac{d_{\mathcal{H}}}{n}}$ for a constant $c_1 < 1$ to be determined.

\textbf{Distribution} $\mu_0$: We define $\mu_0$ on $\mathcal{S}$ as follows: $
    \mu_0(\{x_{i}\})=\mu_0(\{x'_{i}\})=\frac{2\alpha}{d} \ \ \text{for} \ \ i=1,...,\frac{d}{2}
$, and $\mu_0(\{x_0\})=1-2\alpha$.

\textbf{Distribution} $\mu_{1}^{\sigma}$: We define $\mu_{1}^{\sigma}$ on $\mathcal{S}$ as follows: $\mu_{1}^{\sigma}(\{x_{i}\})=\frac{1}{d}+(\sigma_i/2)\cdot\frac{\Delta}{d}$ and $\mu_{1}^{\sigma}(\{x'_{i}\})=\frac{1}{d}-(\sigma_i/2)\cdot\frac{\Delta}{d}$ for $i=1,...,\frac{d}{2}$, and $\mu_{1}^{\sigma}(\{x_0\})=0$.

\textbf{Reduction to a packing.} Any classifier $\hat{h}:\mathcal{S}\rightarrow \{0,1\}$ can be reduced to a binary sequence in the domain $\{-1,+1\}^{d}$. We can first map $\hat{h}$ to $(\hat{h}(x_{1}),\hat{h}(x_{2}),...,\hat{h}(x_{\frac{d}{2}}),\hat{h}(x'_{1}),...,\hat{h}(x'_{\frac{d}{2}}))$ and then convert any element $0$ to $-1$.
We choose the Hamming distance as the distance required in Theorem \ref{tsybakov}. By applying Proposition \ref{gilbert} we can get a subset $\Sigma$ of $\{-1,+1\}^{\frac{d}{2}}$ with $|\Sigma|=M\geq 2^{d/16}$ such that the hamming distance of any two $\sigma,\sigma'\in\Sigma$ is at least $d/16$. Any $\sigma,\sigma'\in \Sigma$ can be mapped to binary sequences in the domain $\{+1,-1\}^d$ by replicating and negating, i.e., $(\sigma,-\sigma),(\sigma',-\sigma')\in\{+1,-1\}^d$ and the hamming distance of resulting sequences in the domain $\{+1,-1\}^d$ is at least $d/8$. Then for any $\hat{h}\in \tilde{\mathcal{H}}$ with $R_{\mu_0}(\hat{h})\leq  \alpha+\frac{\alpha}{d}$ and $\sigma\in \Sigma$, if the hamming distance of the corresponding binary sequence of $\hat{h}$ and $\sigma$ in the space $\{+1,-1\}^d$ is at least $d/8$ then we have $R_{\mu_1}(\hat{h})-R_{\mu_1}(h_{\sigma})\geq \frac{d}{8}\cdot\frac{\Delta}{d}=\frac{\Delta}{8}$. In particular, for any $\sigma,\sigma'\in\Sigma$ we have $R_{\mu_1}(h_{\sigma'})-R_{\mu_1}(h_{\sigma})\geq \frac{d}{8}\cdot\frac{\Delta}{d}=\frac{\Delta}{8}.$

\textbf{KL divergence bound.} Define $\Pi_{\sigma}=(\mu_{1}^{\sigma})^{n}$. For any $\sigma,\sigma'\in \Sigma$, we bound the KL divergence of $\Pi_{\sigma}, \Pi_{\sigma'}$ as follows. We have 
\begin{align*}
    \mathcal{D}_{kl}(\Pi_{\sigma}|\Pi_{\sigma'})=n\cdot \mathcal{D}_{kl}(\mu_{1}^{\sigma}|\mu_{1}^{\sigma'})
\end{align*}

The distribution $\mu_{1}^{\sigma}$ can be expressed as $P^{\sigma}_{X}\times P^{\sigma}_{Y|X}$ where $P^{\sigma}_{X}$ is a uniform distribution over the set $\{1,2,...,\frac{d}{2}\}$ and $P^{\sigma}_{Y|X=i}$ is a Bernoulli distribution with parameter $\frac{1}{2}+\frac{1}{2}\cdot(\sigma_i/2)\cdot\Delta$. Hence we get
\begin{align*}
     \mathcal{D}_{kl}(\mu_{1}^{\sigma}|\mu_{1}^{\sigma'})&=\sum_{i=1}^{\frac{d}{2}}\frac{1}{d/2}\cdot \mathcal{D}_{kl}\bigg(\text{Ber}(\frac{1}{2}+\frac{1}{2}\cdot(\sigma_i/2)\cdot\Delta)|\text{Ber}(\frac{1}{2}+\frac{1}{2}\cdot(\sigma'_i/2)\cdot\Delta)\bigg)\\
     &\leq c_0\cdot \frac{1}{4}\cdot\Delta^{2}\\
     &\leq \frac{1}{4}c_0c_1^{2}\cdot\frac{d_{\mathcal{H}}}{n}\\
     &\leq c_0c_1^{2}\cdot \frac{d}{n}
\end{align*}
for some numerical constant $c_0$. Therefore, we obtain $\mathcal{D}_{kl}(\Pi_{\sigma}|\Pi_{\sigma'})\leq c_0c_1^2d$. Then, for sufficiently small $c_1$ we get $\mathcal{D}_{kl}(\Pi_{\sigma}|\Pi_{\sigma'})\leq\frac{1}{8}\log{M}$ which satisfies condition (ii) in Proposition \ref{tsybakov}. Consequently, there exists $\mu_0$ with $\mathcal{H}_{\alpha}(\mu_0)=\mathcal{H}_{\alpha+\epsilon_0}(\mu_0)$, where $0<\epsilon_0\leq \frac{\alpha}{d_{\mathcal{H}}}$, such that for any randomized $(\alpha,\delta_0)$-learner, we get the following inequality on the conditional probability: \begin{align*}
\sup_{\mathcal{U}(\mu_0)}\mathbb{P}\big(\mathcal{E}_{\alpha}(\hat{h})>c\cdot \sqrt{\frac{d_{\mathcal{H}}}{n}}\ \big|\  \hat{h}\in\mathcal{H}_{\alpha+\epsilon_0}(\mu_0)\big)\geq c',
\end{align*}
which implies that $$
\sup_{\mathcal{U}(\mu_0)}\mathbb{P}\big(\mathcal{E}_{\alpha}(\hat{h})\cdot\indicator{\hat{h}\in \mathcal{H}_{\alpha+\epsilon_0}(\mu_0)}>c\cdot \sqrt{\frac{d_{\mathcal{H}}}{n}}\big)\geq (1-\delta_0)c'\geq c''
$$
for some numerical constants $c,c',c''$.

\subsection*{Case II: $\mathbf{d_{\mathcal{H}}< 17}$.}
Since $\mathcal{H}$ separates three points, we can pick three points $\mathcal{S}=\{x_{0},x_{1},x_{2}\}$ from $\mathcal{X}$ satisfying three-points-separation. Fix $\Delta=c_1\cdot\sqrt{\frac{1}{n}}$ for a constant $c_1<1$ to be determined. Then, we construct a distribution $\mu_0$ and two distributions $\{\mu_{1}^k\}$ for $k=-1,1$.

\textbf{Distribution} $\mu_0$: We define $\mu_0$ on $\mathcal{S}$ as follows: $\mu_0(\{x_0\})=1-2\alpha$ and $\mu_0(\{x_1\})=\mu_0(\{x_2\})=\alpha$.

\textbf{Distribution} $\mu_{1}^{k}$: We define $\mu_{1}^{k}$ on $\mathcal{S}$ as follows: $\mu_{1}^{k}(\{x_1\})=\frac{1}{2}+\frac{k}{2}\cdot \Delta, \mu_{1}^{k}(\{x_2\})=\frac{1}{2}-\frac{k}{2}\cdot \Delta$, and $\mu_{1}^{k}(\{x_0\})=0$.

Let $\Pi_{k}=(\mu_{1}^k)^{n}$ for $k=-1,1$. Then using the same argument we get $\mathcal{D}_{kl}(\Pi_{-1}|\Pi_{1})\leq c$ where $c$ is a numerical constant. Furthermore, let $h_k$ be the hypothesis with $\mu_0$-risk at most $\alpha$ that achieves lowest $\mu_1$-risk with respect to the distributions $(\mu_0,\mu_{1}^k)$. Then we have $R_{\mu_{1}^k}(h_{-k})-R_{\mu_{1}^k}(h_{k})=\epsilon$. Using Le Cam's method [see, e.g. Section 15.2 of \cite{wainwright2019high} ] we conclude that there exists a distribution $\mu_0$ with $\mathcal{H}_{\alpha}(\mu_0)=\mathcal{H}_{\alpha+\epsilon_0}(\mu_0)$, where $0<\epsilon_0\leq \frac{\alpha}{3}$, such that for any randomized $(\alpha,\delta_0)$-learner $\hat{h}$, the following inequality on the conditional probability holds:
$$\sup_{\mathcal{U}(\mu_0)}\mathbb{P}\big(\mathcal{E}_{\alpha}(\hat{h})>c\cdot \sqrt{\frac{1}{n}}\ \big|\ \hat{h}\in \mathcal{H}_{\alpha+\epsilon_0}(\mu_0) \big)\geq c'$$
which implies that $$
\sup_{\mathcal{U}(\mu_0)}\mathbb{P}\big(\mathcal{E}_{\alpha}(\hat{h})\cdot \indicator{\hat{h}\in \mathcal{H}_{\alpha+\epsilon_0}(\mu_0)}>c\cdot \sqrt{\frac{1}{n}}\big)\geq (1-\delta_0)c'\geq c''
$$ for some numerical constants $c,c',c''$. Since $d_{\mathcal{H}}\leq 16$ we conclude that 
$$
\sup_{\mathcal{U}(\mu_0)}\mathbb{P}\big(\mathcal{E}_{\alpha}(\hat{h})\cdot \indicator{\hat{h}\in \mathcal{H}_{\alpha+\epsilon_0}(\mu_0)}>c\cdot \sqrt{\frac{d_{\mathcal{H}}}{n}}
\big)\geq c''$$.
\end{proof}
\begin{proof}[Proof of Lemma \ref{lower} (Lower-Bounds when $\mathcal{H}$ does not separate three points)]
Since $\mathcal{H}_{\alpha}(\mu_0)$ does not admit a maximal element, it must have more than an element. Therefore, by Lemma \ref{total}, $\mathcal{H}_{\alpha}(\mu_0)$ is totally ordered. Consider an $h_0\in\mathcal{H}_{\alpha}(\mu_0)$ such that $A_0=\{h_0=1\}$ is not a null set, and take an element $x_0\in A_0$. Moreover, define the distribution $\sigma_0$ such that its support is the point $\{x_0\}$, i.e., $\sigma_0(x_0)=1$.

Then, let $\hat{h}$ be any randomized $(\alpha,\delta_0)$-learner. Therefore, with probability $1-\delta_0$, the learner returns a function from $\mathcal{H}_{\alpha}(\mu_0)$, i.e., $\hat{h}:\mathcal{X}^{n}\rightarrow \mathcal{H}_{\alpha}(\mu_0)$. Furthermore, let $\hat{h}_0=\hat{h}(\{x_0\}^{n})$. Since $\mathcal{H}_{\alpha}(\mu_0)$ does not admit a maximal element, there exists an $h_1\in \mathcal{H}_{\alpha}(\mu_0)$ such that $A_0\cup \{\hat{h}_0=1\}\subsetneq \{h_1=1\}.$
Then, define the distribution $\sigma_1$ supported on $x_1\in\{h_1=1\}\setminus( \{\hat{h}_0=1\}\cup A_0)$, and let the distribution $\mu_1$ be $\mu_1=(1-\frac{1}{n})\sigma_0+\frac{1}{n}\sigma_1.$ Subsequently, we obtain $\mathcal{E}_{\alpha}(\hat{h}_0)=R_{\mu_1}(\hat{h}_0)\geq \frac{1}{n}.$ Hence, we get
\begin{align*}
    \underset{X^{n}\sim\mu_1^{n}}{\mathbb{P}}(\hat{h}(X^{n})=\hat{h}_0)&\geq \underset{X^{n}\sim\mu_1^{n}}{\mathbb{P}}(\hat{h}(X^{n})=\hat{h}_0 \ \text{and}\ X^{n}=\{x_0\}^{n})\\
    &\geq \int_{X^{n}=\{x_0\}^{n}} \mathbb{1}_{\{\hat{h}(X^{n})=\hat{h}_0\}} \Pi_{j=1}^{n} d\mu_1(X_j)\\
    &\geq (1-\frac{1}{n})^{n} \int_{X^{n}=\{x_0\}^{n}} \mathbb{1}_{\{\hat{h}(X^{n})=\hat{h}_0\}} \Pi_{j=1}^{n} d\sigma_0(X_j)\\
    &=(1-\frac{1}{n})^{n} \underset{X^{n}\sim \sigma_0^{n}}{\mathbb{P}}(\hat{h}(X^{n})=\hat{h}_0)\\
    &=(1-\frac{1}{n})^{n}\\
    &\geq \frac{1}{2e}
\end{align*}
Therefore, condition on the event $\hat{h}\in \mathcal{H}_{\alpha}(\mu_0)$, we obtain
\begin{align*}
    \underset{X^{n_1}\sim\mu_1^{n}}{\mathbb{P}}(\mathcal{E}_{\alpha}(\hat{h}(X^{n}))\geq\frac{1}{n})&\geq \underset{X^{n}\sim\mu_1^{n}}{\mathbb{P}}(\mathcal{E}_{\alpha}(\hat{h}(X^{n}))\geq\frac{1}{n}| \hat{h}(X^{n})=\hat{h}_0)\underset{X^{n}\sim\mu_1^{n}}{\mathbb{P}}(\hat{h}(X^{n})=\hat{h}_0)
    \geq\frac{1}{2e}.
\end{align*}

Hence the unconditional probability is as follows
\begin{align*}
   \underset{X^{n}\sim\mu_1^{n}}{\mathbb{P}}(\mathcal{E}_{\alpha}(\hat{h}(X^{n}))\geq\frac{1}{n})\geq \frac{1-\delta_0}{2e}\geq c 
\end{align*}
for some numerical constant c.
\end{proof}
\begin{proof}[Proof of Lemma \ref{lemma4.8} (Constructing a new measure)]
    Without loss of generality, we only consider a distribution $\mu_0$ such that its mass in the region $\{h=1\}$, $h\in \mathcal{H}_{\alpha+\epsilon_0}(\mu_0)\setminus \mathcal{H}_{\alpha}(\mu_0)$, is at most $\epsilon_0$ (we can ensure this condition by simply moving any excess mass in this region to the region $\{h=1\}$, $h\in \mathcal{H}_{\alpha}(\mu_0)$). Moreover, let $\tilde{\mathcal{X}}=\text{supp}(\mu_0)$ and $\tilde{\mathcal{H}}$ denote the projection of $\mathcal{H}$ onto $\tilde{\mathcal{X}}$, i.e., two hypotheses $h,h'$ are equivalent iff $\{h=1\}\cap \tilde{\mathcal{X}}=\{h'=1\}\cap \tilde{\mathcal{X}}$. Note that for any measure $\mu$ such that $\text{supp}(\mu)\subset \tilde{\mathcal{X}}$ and any $h\in \mathcal{H}$, we have $\mu(\{h=1\})=\mu(\{\tilde{h}=1\})$

    If $\tilde{\mathcal{H}}_{\alpha}(\mu_0)=\tilde{\mathcal{H}}_{\alpha+\epsilon_0}(\mu_0)$, then $\mu_0$ would have the desired property and define $\mu_0':=\mu_0$. Otherwise, we consider the following two cases separately (I) $\tilde{\mathcal{H}}_{\alpha+\epsilon_0}(\mu_0)\neq \tilde{\mathcal{H}}_{\alpha+2\epsilon_0}(\mu_0)$ and (II) $\tilde{\mathcal{H}}_{\alpha+\epsilon_0}(\mu_0)= \tilde{\mathcal{H}}_{\alpha+2\epsilon_0}(\mu_0)$, and construct $\mu'_0$ such that $\mathcal{H}_{\alpha}(\mu'_0)=\mathcal{H}_{\alpha+\epsilon_0}(\mu'_0)=\mathcal{H}_{\alpha+\epsilon_0}(\mu_0)$ as follows.

    \textbf{Case I:} There exists a point $w\in \tilde{\mathcal{X}}$ such that for any $h\in \tilde{\mathcal{H}}_{\alpha+2\epsilon_0}(\mu_0)\setminus \tilde{\mathcal{H}}_{\alpha+\epsilon_0}(\mu_0)$ we have $w\in \{h=1\}$. Then we construct the new measure $\mu_0'$ from $\mu_0$ simply by moving the mass in the region $\{h=1\}$, $h\in \tilde{\mathcal{H}}_{\alpha+\epsilon_0}(\mu_0)\setminus \tilde{\mathcal{H}}_{\alpha}(\mu_0)$, to the point $w$.

    \textbf{Case II:} Take a point $z\in\tilde{\mathcal{X}}$ such that $z$ lies outside the region $\{h=1\}$, $h\in \tilde{\mathcal{H}}_{\alpha+2\epsilon_0}(\mu_0)$, and then construct the new measure $\mu'_0$ from $\mu_0$ simply by moving the mass in the region $\{h=1\}$, $h\in \tilde{\mathcal{H}}_{\alpha+\epsilon_0}(\mu_0)\setminus \tilde{\mathcal{H}}_{\alpha}(\mu_0)$, to the point $z$. Note that for any $h\notin \tilde{\mathcal{H}}_{\alpha+2\epsilon_0}(\mu_0)$, we have $h\notin \tilde{\mathcal{H}}_{\alpha+\epsilon_0}(\mu'_0)$ regardless of whether $\{h=1\}$ contains the point $z$ or not.  
\end{proof}

\end{document}